%% file: neurips_2023.tex
\documentclass[dvipsnames]{article}
\pdfoutput=1

\PassOptionsToPackage{numbers,sort}{natbib}

\usepackage[dvipdfmx]{graphicx}
\usepackage{bmpsize}
\usepackage[final]{neurips_2023}

\usepackage[utf8]{inputenc} 
\usepackage[T1]{fontenc}    
\usepackage{url}            
\usepackage{booktabs}       
\usepackage{amsfonts}       
\usepackage{amssymb}
\usepackage{nicefrac}       
\usepackage{microtype}      
\usepackage[dvipsnames]{xcolor}        
\usepackage{multicol, multirow}
\usepackage{tcolorbox, bbm}
\usepackage{abraces}
\usepackage{wrapfig}
\usepackage{tikz}
\usepackage{sidecap}
\usepackage{esvect}
\usepackage{algorithm}
\usepackage{algpseudocode}
\usepackage[symbol]{footmisc}

\usepackage{varioref}
\usepackage{xr-hyper}
\usepackage[pagebackref,breaklinks,colorlinks,citecolor=blue]{hyperref}

\makeatletter
\newcommand*{\addFileDependency}[1]{
  \typeout{(#1)}
  \@addtofilelist{#1}
  \IfFileExists{#1}{}{\typeout{No file #1.}}
}
\makeatother

\hypersetup{
colorlinks = true,
linkcolor = RoyalBlue,
anchorcolor = blue,
citecolor = Blue,
filecolor = cyan,
menucolor = ForestGreen,
runcolor = cyan,
urlcolor = RoyalBlue}

\include{default_includes}

\include{math_commands.tex}

\include{paper_macros}

\title{Gaussian Membership Inference Privacy}

\author{%
  Tobias Leemann$^{*}$ \\ 
  University of Tübingen \\
  Technical University of Munich \\
  \And
  Martin Pawelczyk$^{*}$ \\
  Harvard University \\
  \And
  Gjergji Kasneci \\
  Technical University of Munich \\
}

\begin{document}
\maketitle

\begin{abstract}
\input{neurips/0-abtract}
\end{abstract}


\renewcommand{\thefootnote}{\arabic{footnote}}
\section{Introduction}
\input{neurips/1-introduction}

\section{Related Work}

\input{2-related_work}

\section{Preliminaries}
\input{neurips/3-preliminaries}

\section{Navigating Between Membership Inference Privacy and DP}
\input{neurips/4-fMIP}

\section{Experimental Evaluation}
\input{6-experiments}

\section{Conclusion and Future Work}
\input{7-conclusion}

\bibliographystyle{neurips2023}
\bibliography{neurips_2023}

\newpage
\appendix
\input{8-appendix}

\end{document}

%% file: default_includes.tex
\usepackage{amsmath, mathtools, amssymb, amsthm}
\usepackage[capitalise]{cleveref}

\usepackage{subcaption, caption}
\usepackage{bm}
\usepackage{adjustbox}
\usepackage{tikz}
\usepackage{wrapfig}

\usetikzlibrary{calc,matrix,positioning,patterns,shapes.geometric}

\newtheorem{definition}{Definition}[section]
\newtheorem{corollary}{Corollary}[section]
\newtheorem{lemma}{Lemma}[section]
\newtheorem{remark}{Remark}[section]

\newtheorem{theorem}{Theorem}[section]
\crefname{condition}{condition}{conditions}
\Crefname{condition}{Condition}{Conditions}
\crefname{example}{example}{example}
\Crefname{example}{Example}{Example}
\Crefname{section}{Section}{Section} 
\crefname{section}{Sec.}{Sec.} 
\crefname{figure}{Fig.}{Figs.} 
\Crefname{figure}{Figure}{Figures} 
\Crefname{table}{Table}{Tables} 
\crefname{table}{Tab.}{Tab.} 
\Crefname{equation}{Equation}{Equations} 
\crefname{equation}{Eqn.}{Eqns.} 


%% file: math_commands.tex
\newcommand{\A}{\mathcal{A}}

\newcommand{\D}{\mathcal{D}}

\newcommand{\btheta}[0]{\hat{\bm{\theta}}}

\def\vmu{{\bm{\mu}}}
\def\vdelta{{\bm{\delta}}}
\def\vtheta{{\bm{\theta}}}

\def\vb{{\bm{b}}}

\def\vd{{\bm{d}}}

\def\vg{{\bm{g}}}

\def\vl{{\bm{l}}}
\def\vm{{\bm{m}}}

\def\vs{{\bm{s}}}

\def\vw{{\bm{w}}}
\def\vx{{\bm{x}}}

\def\vtheta{{\bm{\theta}}}
\def\vmu{{\bm{\mu}}}

\def\mD{{\bm{D}}}

\def\mI{{\bm{I}}}

\def\mQ{{\bm{Q}}}

\def\mX{{\bm{X}}}

\def\mLambda{{\bm{\Lambda}}}
\def\mSigma{{\bm{\Sigma}}}

\def\mTheta{{\bm{\Theta}}}

%% file: paper_macros.tex
\newcommand{\adversary}{\ensuremath{\A}}

%% file: neurips/0-abtract.tex
We propose a novel and practical privacy notion called $f$-Membership Inference Privacy ($f$-MIP), which explicitly considers the capabilities of realistic adversaries under the membership inference attack threat model. 
Consequently, $f$-MIP offers interpretable privacy guarantees and improved utility (e.g., better classification accuracy).
In particular, we derive a parametric family of $f$-MIP guarantees that we refer to as $\mu$-Gaussian Membership Inference Privacy ($\mu$-GMIP) by theoretically analyzing likelihood ratio-based membership inference attacks on stochastic gradient descent (SGD).
Our analysis highlights that models trained with standard SGD already offer an elementary level of MIP. 
Additionally, we show how $f$-MIP can be amplified by adding noise to gradient updates.
Our analysis further yields an analytical membership inference attack that offers two distinct advantages over previous approaches. 
First, unlike existing state-of-the-art attacks that require training hundreds of shadow models, our attack does not require \emph{any} shadow model. 
Second, our analytical attack enables straightforward auditing of our privacy notion $f$-MIP.
Finally, we quantify how various hyperparameters (e.g., batch size, number of model parameters) and specific data characteristics determine an attacker's ability to accurately infer a point's membership in the training set.
We demonstrate the effectiveness of our method on models trained on vision and tabular datasets. \footnotetext{$^{*}$Equal contribution. Corresponding authors: \texttt{tobias.leemann@uni-tuebingen.de} and \texttt{martin.} \texttt{pawelczyk.1@gmail.com}.}

%% file: neurips/1-introduction.tex
Machine learning (ML) has seen a surge in popularity and effectiveness, leading to its widespread application across various domains. 
However, some of these domains, such as finance and healthcare, deal with sensitive data that cannot be publicly shared due to ethical or regulatory concerns. 
Therefore, ensuring data privacy becomes crucial at every stage of the ML process, including model development and deployment.
In particular, the trained model itself \citep{shokri2017membership,carlini2021membership} or explanations computed to make the model more interpretable \citep{shorki2021explanation,pawelczyk2022privacy} may leak information about the training data if appropriate measures are not taken. 
For example, this is a problem for recent generative Diffusion Models 
\citep{carlini2023extracting} and Large Language models, where the data leakage seems to be amplified by model size \citep{carlini2021extracting}.

Differential privacy (DP) \citep{dwork2006calibrating} is widely acknowledged as the benchmark for ensuring provable privacy in academia and industry \citep{cummings2023challenges}. 
DP utilizes randomized algorithms during training and guarantees that the output of the algorithm will not be significantly influenced by the inclusion or exclusion of any individual sample in the dataset. 
This provides information-theoretic protection against the maximum amount of information that an attacker can extract about any specific sample in the dataset, even when an attacker has full access to and full knowledge about the predictive model.

While DP is an appealing technique for ensuring privacy, DP's broad theoretical guarantees often come at the expense of a significant loss in utility for many ML algorithms.
This utility loss cannot be further reduced by applying savvier algorithms: Recent work \citep{nasr2021adversary,nasr2023tight} confirms that an attacker can be implemented whose empirical capacity to differentiate between neighboring datasets $D$ and $D'$ when having access to privatized models matches the theoretical upper bound.
This finding suggests that to improve a model's utility, we need to take a step back and inspect the premises underlying DP.
For example, previous work has shown that privacy attacks are much weaker when one imposes additional realistic restrictions on the attacker's capabilities \citep{nasr2021adversary}. 

In light of these findings, we revisit the DP threat model and identify three characteristics of an attacker that might be overly restrictive in practice.
First, DP grants the attacker full control over the dataset used in training including the capacity to poison all samples in the dataset.
For instance,  DP's protection includes pathological cases such as an empty dataset and a dataset with a single, adversarial instance \citep{nasr2023tight}. 
Second, in many applications, it is more likely that the attacker only has access to an API to obtain model predictions \citep{shokri2017membership,dwork2018privacy} or to model gradients \citep{kairouz2021advances}.
Finally, one may want to protect typical samples from the data distribution.
As argued by \citet{triastcyn2020bayesian}, it may be over-constraining to protect images of dogs in a model that is conceived and trained with images of cars. 

Such more realistic attackers have been studied in the extensive literature on Membership Inference (MI) attacks (e.g., \citep{yeom2018privacy, carlini2021membership}), where the attacker attempts to determine whether a sample from the data distribution was part of the training dataset. 
Under the MI threat model, \citet{carlini2021membership} observe that ML models with very lax ($\epsilon > 5000$) or no ($\epsilon = \infty$) DP-guarantees still provide some defense against membership inference attacks \citep{yeom2018privacy, carlini2021membership}.
Hence, we hypothesize that standard ML models trained with low or no noise injection may already offer some level of protection against realistic threats such as MI, despite resulting in very large provable DP bounds.

To build a solid groundwork for our analysis, we present a hypothesis testing interpretation of MI attacks.
We then derive $f$-Membership Inference Privacy ($f$-MIP), which bounds the trade-off between an MI attacker's false positive rate (i.e., FPR, type I errors) and false negative rate (i.e., FNR, type II errors) in the hypothesis testing problem by some function $f$.
We then analyze the privacy leakage of a gradient update step in stochastic gradient descent (SGD) and derive the first analytically optimal attack based on a likelihood ratio test.
However, for $f$-MIP to cover practical scenarios, post-processing and composition operations need to be equipped with tractable privacy guarantees as well.
Using $f$-MIP's handy composition properties, we analyze full model training via SGD and derive explicit $f$-MIP guarantees.
We further extend our analysis by adding carefully calibrated noise to the SGD updates to show that $f$-MIP may be guaranteed without any noise or with less noise than the same parametric level of $f$-DP \citep{dong2022gaussian}, leading to a smaller loss of utility.

Our analysis comes with a variety of novel insights:
We confirm our hypothesis that, unlike for DP, no noise ($\tau^2=0$) needs to be added during SGD to guarantee $f$-MIP. 
Specifically, we prove that the trade-off curves of a single SGD step converge to the family of Gaussian trade-offs identified by \citet{dong2022gaussian} and result in the more specific $\mu$-Gaussian Membership Inference Privacy ($\mu$-GMIP). 
The main contributions this research offers to the literature on privacy preserving ML include:
\begin{enumerate}
\setlength\itemsep{0.0cm}
\item \textbf{Interpretable and practical privacy notion}: We suggest the novel privacy notion of $f$-MIP that addresses the realistic threat of MI attacks.
$f$-MIP considers the MI attacker's full trade-off curve between false positives and false negatives. Unlike competing notions, $f$-MIP offers appealing composition and post-processing properties.

\item \textbf{Comprehensive theoretical analysis}: We provide (tight) upper bounds on any attacker's ability to run successful MI attacks, i.e., we bound any MI attacker's ability to identify whether points belong to the training set when ML models are trained via gradient updates. 

\item \textbf{Verification and auditing through novel attacks}: 
As a side product of our theoretical analysis, which leverages the Neyman-Pearson lemma, we propose a novel set of attacks for auditing privacy leakages. 
An important advantage of our analytical Gradient Likelihood-Ratio (GLiR)  attack is its computational efficiency. Unlike existing attacks that rely on training hundreds of shadow models to approximate the likelihood ratio, our attack does not require any additional training steps.

\item \textbf{Privacy amplification through noise addition}:
Finally, our analysis shows how one can use noisy SGD (also known as Differentially Private SGD \cite{abadi2016deep}) to reach $f$-MIP while maintaining worst-case DP guarantees. 
Thereby our work establishes a theoretical connection between $f$-MIP and $f$-DP \cite{dong2022gaussian}, which allows to translate an $f$-DP guarantee into an $f$-MIP guarantee and vice versa.
\end{enumerate} 

%% file: 2-related_work.tex
\textbf{Privacy notions.}
DP and its variants provide robust, information-theoretic privacy guarantees by ensuring that the probability distribution of an algorithm's output remains stable even when one sample of the input  dataset is changed \citep{dwork2006calibrating}. 
For instance, a DP algorithm is $\varepsilon$-DP if the probability of the algorithm outputting a particular subset $E$ for a dataset $S$ is not much higher than the probability of outputting $E$ for a dataset $S_0$ that differs from $S$ in only one element.
DP has several appealing features, such as the ability to combine DP algorithms without sacrificing guarantees.

A few recent works have proposed to carefully relax the attacker’s capabilities in order to achieve higher utility from private predictions \citep{bassily2018model,dwork2018privacy,triastcyn2020bayesian,izzo2022provable}. 
For example, \citet{dwork2018privacy} suggest the notion of ``privacy-preserving prediction'' to make private model predictions through an API interface.
Their work focuses on PAC learning guarantees of any class of Boolean functions.
Similarly, \citet{triastcyn2020bayesian} suggest ``Bayesian DP'', which is primarily based on the definition of DP, but restricts the points in which the datasets $S$ and $S_0$ may differ to those sampled from the data distribution. 
In contrast, \citet{izzo2022provable} introduces a notion based on MI attacks, where their approach guarantees that an adversary $\adversary$ does not gain a significant advantage in terms of accuracy when distinguishing whether an element $\vx$ was in the training data set compared to just guessing the most likely option. 
However, they only constrain the accuracy of the attacker, while we argue that it is essential to bound the entire trade-off curve, particularly in the low FPR regime, to prevent certain re-identification of a few individuals \cite{carlini2021membership}.
Our work leverages a hypothesis testing formulation that covers the entire trade-off curve thereby offering protection also to the most vulnerable individuals. 
Additionally, our privacy notion maintains desirable properties such as composition and privacy amplification through subsampling, which previous notions did not consider.

\textbf{Privacy Attacks on ML Models.}
Our work is also related to auditing privacy leakages through a common class of attacks called MI attacks.
These attacks determine if a given instance is present in the training data of a particular model \citep{yeom2018privacy,shokri2017membership,shorki2021explanation,long2018understanding,sablayrolles2019whitebox,haim2022reconstructing,carlini2021membership,carlini2023extracting,pawelczyk2022privacy,tan2022parameters,tan2023blessing,Choquette2019label,ye2021enhanced}.
Compared to these works, our work suggests a new much stronger class of MI attacks that is analytically derived and uses information from model gradients.
An important advantage of our analytically derived attack is its computational efficiency, as it eliminates the need to train any additional shadow models.

%% file: neurips/3-preliminaries.tex

The classical notion of $(\varepsilon,\delta)$-differential privacy \cite{dwork2006calibrating} is the current workhorse of private ML and can be described as follows:
An algorithm is DP if for any two neighboring datasets $S, S'$ (that differ by one instance) and any subset of possible outputs, the ratio of the probabilities that the algorithm's output lies in the subset for inputs $S, S^\prime$ is bounded by a constant factor. DP is a rigid guarantee, that covers \emph{every} pair of datasets $S$ and $S'$, including pathologically crafted datasets (for instance, \citet{nasr2023tight} use an empty dataset) that might be unrealistic in practice.
For this reason, we consider a different attack model in this work: The MI game \citep{yeom2018privacy}. 
This attack mechanism on ML models follows the goal of inferring an individual's membership in the training set of a learned ML model.
We will formulate this problem using the language of hypothesis testing and trade-off functions, a concept from hypothesis testing theory \citep{dong2022gaussian}. 
We will close this section by giving several useful properties of trade-off functions which we leverage in our main theoretical results presented in Sections \ref{sec:navigating} and \ref{sec:implementing_fmip}.



\subsection{Membership Inference Attacks}
The overarching goal of privacy-preserving machine learning lies in protecting personal data. To this end, we will show that an alternative notion of privacy can be defined through the success of a MI attack which attempts to infer whether a given instance was present in the training set or not.
Following \citet{yeom2018privacy} we define the standard MI experiment as follows:
\begin{definition}[Membership Inference Experiment \citep{yeom2018privacy}] Let $\mathcal{A}$ be an attacker, $A$ be a learning algorithm, $N$ be a positive integer, and $\mathcal{D}$ be a distribution over data points $\vx \in D$, where the vector $\vx$ may also be a tuple of data and labels.
The MI experiment proceeds as follows:
The model and data owner $\mathcal{O}$ samples $S \sim \mathcal{D}_N$ (i.e, sample n points i.i.d.\ from $\mathcal{D}$) and trains $A_S = A(S)$.
They choose $b \in \left\{0, 1\right\}$ uniformly at random and draw $\vx' \sim \mathcal{D}$ if $b = 0$, or $\vx' \sim S$ if $b = 1$. Finally, the attacker is successful if $\mathcal{A}(\vx', A_S, N, \mathcal{D}) = b$. $\mathcal{A}$ must output either 0 or 1.
\end{definition}
We note that the membership inference threat model features several key differences to the threat model underlying DP, which are listed in \Cref{tab:dp_mip_comparison}.
Most notably, in MI attacks, the datasets are sampled from the distribution $\mathcal{D}$, whereas DP protects all datasets.
This corresponds to granting the attacker the capacity of full dataset manipulation.
Therefore, the MI attack threat model is sensible in cases where the attacker cannot manipulate the dataset through injection of malicious samples also called ``canaries''.
This may be realistic for financial and healthcare applications, where the data is often collected from actual events (e.g., past trades) or only a handful of people (trusted hospital staff) have access to the records. 
In such scenarios, it might be overly restrictive to protect against worst-case canary attacks as attackers cannot freely inject arbitrary records into the training datasets.
Furthermore, MI attacks are handy as a fundamental ingredient in crafting data extraction attacks \citep{carlini2021extracting}. Hence we expect a privacy notion based on the MI threat model to offer protection against a  broader class of reconstruction attacks.
Finally, being an established threat in the literature \citep{shokri2017membership, yeom2018privacy, carlini2021membership, Choquette2019label, ye2021enhanced}, MI can be audited through a variety of existing attacks.


\begin{table}[htb]
\centering
\resizebox{\columnwidth}{!}{%
\begin{tabular}{c  p{6cm}  p{6cm}}
\toprule
& f-DP threat model & f-MIP threat model (this work)  \\
\cmidrule(lr){1-1} \cmidrule(lr){2-3}
Goal & Distinguish between $S$ and $S'$ for \emph{any} $S,S'$ that differ in at most one instance. & Distinguish whether $\vx' \in S$ (training data set) or not. 
\\
\midrule
Dataset access &  Attacker has full data access. For example, the attacker can poison or adversarially construct datasets on which ML models could be trained; e.g., $S=\{ \}$ and   $S'=\{ 10^6\}$.   & Attacker has no access to the training data set; i.e., the model owner privately trains their model free of adversarially poisoned samples.  \\ 
\midrule
Protected Instances & The instance in which $S$ and $S^\prime$ differ is arbitrary. This includes OOD samples and extreme outliers. & The sample  $\vx'$ for which membership is to be inferred is drawn from the data distribution $\mathcal{D}$. Therefore, MI is concerned with typical samples that can occur in practice. \\
\midrule
Best used & When the specific attack model is unknown. Offers a form of general protection. 	& When dataset access (e.g. canary injection) of an attacker can be ruled out and the main attack goal lies in revealing private training data (e.g., membership inference, data reconstruction). \\
\midrule
Model knowledge & 
\multicolumn{2}{p{12cm}}{The attacker knows the model architecture and has full access to the model in form of its parameters, hyperparameters and its model outputs.} \\
\bottomrule
\end{tabular}
}\vspace{1em}
\caption{Comparing the threat models underlying $f$-DP and $f$-MIP.}
\label{tab:dp_mip_comparison}
\end{table}

\subsection{Membership Inference Privacy as a Hypothesis Testing Problem}
While DP has been studied through the perspective of a hypothesis testing formulation for a while \citep{wasserman2010statistical,kairouz2015composition,balle2020hypothesis,dong2022gaussian}, we adapt this route to formulate membership inference attacks.
To this end, consider the following hypothesis test:
\begin{align}
\text{H}_0: \vx' \notin S  \text{ vs. } \text{H}_1: \vx' \in S. 
\label{eq:hypothis_test}
\end{align}
Rejecting the null hypothesis corresponds to detecting the presence of the individual $\vx'$ in $S$, whereas failing to reject the null hypothesis means inferring that $\vx'$ was not part of the dataset $S$.
The formulation in \eqref{eq:hypothis_test} is a natural vehicle to think about any attacker's capabilities in detecting members of a train set in terms of false positive and true positive rates.
The motivation behind these measures is that the attacker wants to reliably identify the subset of data points belonging to the training set (i.e., true positives) while incurring as few false positive errors as possible \citep{carlini2021membership}.
In other words, the attacker wants to maximize their true positive rate at any chosen and ideally low false positive rate (e.g., 0.001).
From this perspective, the formulation in \eqref{eq:hypothis_test} allows to define membership inference privacy via trade-off functions $f$ which exactly characterize the relation of false negative rates (i.e., 1-TPR) and false positive rates that an optimal attacker can achieve.
\begin{definition}(Trade-off function \cite{dong2022gaussian})
For any two probability distributions $P$ and $Q$ on the same
space, denote  the trade-off function $\text{Test}(P;Q) : [0; 1] \rightarrow [0; 1]$
\begin{align}
    \text{\normalfont  Test}(P;Q) (\alpha)= \inf \left\{\text{FNR}~ \middle|~ \text{FPR}=\alpha\right\},
\end{align}
where the infimum is taken over all (measurable) rejection rules (``tests'') which lead to a FPR of $\alpha$ between distributions $P, Q$.
\end{definition}
Not every function makes for a valid trade-off function. Instead, trade-off functions possess certain characteristics that are handy in their analysis.
\begin{definition}[Characterization of trade-off functions \citep{dong2022gaussian}]
A function $f : [0, 1] \rightarrow [0, 1]$ is a trade-off function if $f$ is convex, continuous at zero, non-increasing, and $f(r) \leq 1 - r$ for $r \in [0,1]$.
\end{definition}
We additionally introduce a semi-order on the space of trade-off functions to make statements on the hardness of different trade-offs in relation to each other.
\begin{definition}[Comparing trade-offs] A trade-off function $f$ is uniformly at least as hard as another trade-off function $g$, if $f(r) \geq g(r)$ for all $0 \leq r \leq 1$. We write $f \geq g$.
\end{definition}
If $\text{Test}(P; Q) \geq \text{Test}(P^\prime; Q^\prime)$,  testing $P$ vs $Q$ is uniformly at least as hard as testing $P^\prime$ vs $Q^\prime$. 
Intuitively, this means that for a given FPR $\alpha$, the best test possible test on $(P; Q)$ will result in an equal or higher FNR than the best test on $(P^\prime; Q^\prime)$.

\subsection{Noisy Stochastic Gradient Descent (Noisy SGD)}
Most recent large-scale ML models are trained via stochastic gradient descent (SGD).
Noisy SGD (also known as DP-SGD) is a variant of classical SGD that comes with privacy guarantees.
We consider the algorithm as in the work by \citet{abadi2016deep}, which we restate for convenience in \Cref{app:algorithm}.
While its characteristics with respect to DP have been extensively studied, we take a fundamentally different perspective in this work and study the capabilities of this algorithm to protect against membership inference attacks. 
In summary, the algorithm consists of three fundamental steps: 
\emph{gradient clipping} (i.e., $\vtheta_i \coloneqq \vg(\vx_i, y_i) \cdot\max(1, C/\lVert\vg(\vx_i, y_i)\rVert)$ where $\vg(\vx_i, y_i) =\nabla \mathcal{L}(\vx_i, y_i)$ is the gradient with respect to the loss function $\mathcal{L}$),
\emph{aggregation} (i.e., $\vm = \frac{1}{n}\sum_{i=1}^n \vtheta_i$) and
\emph{adding Gaussian noise} (i.e., $\tilde{\vm} = \vm + Y$ where $Y \sim \mathcal{N}(\mathbf{0}, \tau^2\mI)$ with variance parameter $\tau^2$).
To obtain privacy bounds for this algorithm, we study MI attacks for means of random variables.
This allows us to bound the MI vulnerability of SGD.

%% file: neurips/4-fMIP.tex
\label{sec:navigating}
In this section, we formally define our privacy notion $f$-MIP. To this end, it will be handy to view MI attacks as hypothesis tests.

\subsection{Membership Inference Attacks from a Hypothesis Testing Perspective}
Initially, we define the following distributions of the algorithm's output
\begin{align}
\label{eq:hypothesisformulation}
A_0 = A(\mX \cup \left\{\vx\right\}) \text{ with } \mX \sim \mathcal{D}^{n-1}, \vx \sim \mathcal{D} \text{ and } 
A_1(\vx^{\prime}) = A(\mX \cup \left\{\vx^{\prime}\right\}) \text{ with } \mX \sim \mathcal{D}^{n-1},
\end{align}
where we denote other randomly sampled instances that go into the algorithm by $\mX=\left\{\vx_1, ... \vx_{n-1}\right\}$. Here $A_0$ represents the output distribution under the null hypothesis ($H_0$) where the sample $\vx^\prime$ is not part of the training dataset. 
On the other hand, $A_1$ is the output distribution under the alternative hypothesis ($H_1$) where  $\vx^\prime$ was part of the training dataset. The output contains randomness due to the instances drawn from the distribution $\mathcal{D}$ and due to potential inherent randomness in $A$.

We observe that the distribution $A_1$ depends on the sample $\vx^{\prime}$ which is known to the attacker. 
The attacker will have access to samples for which $A_0$ and $A_1(\vx^\prime)$ are simpler to distinguish and others where the distinction is harder. 
To reason about the characteristics of such a stochastically composed test, we define a composition operator that defines an optimal test in such a setup. 
To obtain a global FPR of $\alpha$, an attacker can target different FPRs $\bar{\alpha}(\vx^{\prime})$ for each specific test. 
We need to consider the optimum over all possible ways of choosing $\bar{\alpha}(\vx^{\prime})$, which we refer to as $\textit{test-specific FPR function}$, giving rise to the following definition.
\begin{definition}[Stochastic composition of trade-off functions]\label{def:stochastic_composition}
Let $\mathcal{F}$ be a family of trade-off functions, $h: D \subset \mathbb{R}^d \rightarrow \mathcal{F}$ be a function that maps an instance of the data domain to a corresponding trade-off function, and $\mathcal{D}$ be a probability distribution on $D$.
The set of valid test-specific FPR functions $\bar{\alpha}: D \rightarrow [0,1]$ that result in  a global FPR of $\alpha \in [0,1]$ given the distribution $\mathcal{D}$ is defined through
\begin{align}
    \mathcal{E}(\alpha, \mathcal{D}) = \left\{\bar{\alpha}: D\rightarrow \left[0,1\right] ~\middle|~
 \mathbb{E}_{\vx^\prime \sim \mathcal{D}}\left[\bar{\alpha}(\vx^\prime)\right] =\alpha\right\}.
\end{align}
For a given test-specific FPR function, $\bar{\alpha}$ the global false negative rate (type II error) $\beta$ is given by 
\begin{align}
    \beta_h(\bar{\alpha}) = \mathbb{E}_{\vx^\prime \sim \mathcal{D}} \left[h(\vx)(\bar{\alpha}(\vx))\right],
\end{align}
where $\bar{\alpha}(\vx)$ is the argument to the trade-off function $h(\vx) \in \mathcal{F}$.
For a global $\alpha \in [0,1]$ the stochastic composition of these trade-functions is defined as
\begin{align}
\left(\bigotimes_{\vx \sim \mathcal{D}} h(\vx) \right)(\alpha) = \min_{\bar{\alpha} \in \mathcal{E}(\alpha, \mathcal{D})} \left\{ \beta_h(\bar{\alpha})\right\},
\end{align}
(supposing the minimum exists), the smallest global false negative rate possible at a global FPR of $\alpha$.
\end{definition}
This definition specifies the trade-off function $\bigotimes_{\vx \sim \mathcal{D}} h(\vx): [0,1] \rightarrow [0,1]$ of a stochastic composition of several trade-offs. 
While it is reminiscent of the ``most powerful test'' (MPT) \citep{neyman1933ix}, there are several differences to the MPT that are important in our work.
Most prominently, a straightforward construction of the MPT to MI problems does not work since the adversary does not only run one hypothesis test to guess whether one sample belongs to the training data set or not; instead, the adversary draws multiple samples and runs sample-dependent and (potentially) different hypotheses tests for each drawn sample.
This is necessary due to the form of the alternative hypotheses in the formulation of the test in \eqref{eq:hypothesisformulation}, which depends on the sample $\vx^\prime$.
We therefore require a tool to compose the results from different hypothesis tests. 
Finally, we prove that the trade-off of the stochastic composition has the properties of a trade-off function (see App.\ \ref{app_sec:stochasticcomposition}):
\begin{theorem}[Stochastic composition of trade-off functions]
The stochastic composition $\bigotimes_{\vx \sim \mathcal{D}} h(\vx)$ of trade-off functions $h(\vx)$ maintains the characteristics of a trade-off function, i.e., it is convex, non-increasing, $\left(\bigotimes_{\vx \sim \mathcal{D}} h(\vx)\right)(r) \leq 1 - r$ for all $r \in [0,1]$, and it is continuous at zero.
\label{theorem:stochastic_composition}
\end{theorem}

\subsection{$f$-Membership Inference Privacy ($f$-MIP)}
This rigorous definition of the stochastic composition operator allows us to define membership inference privacy from a hypothesis testing perspective.
\begin{definition}[$f$-Membership Inference Privacy]
\label{def:f-mip}
Let $f$ be a trade-off
function. An algorithm\footnote{When using the term ``algorithm'', we also include randomized mappings.} $A: D^{n} \rightarrow \mathbb{R}^d$ is said to be
$f$-membership inference private ($f$-MIP) with respect to a data distribution $\mathcal{D}$ if
\begin{align}
\bigotimes_{\vx^{\prime} \sim \mathcal{D}}  \text{\normalfont~Test}\left(A_0; A_1(\vx^{\prime}) \right) \geq f,
\end{align}
 where $\vx^{\prime} \sim \mathcal{D}$ and $\bigotimes$ denotes the stochastic composition built from individual trade-off functions of the MI hypotheses tests for random draws of $\vx^{\prime}$.
\end{definition}
In this definition, both sides are functions dependent on the false positive rate $\alpha$. A prominent special case of a trade-off function is the Gaussian trade-off, which stems from testing one-dimensional normal distributions of unit variance that are spaced apart by $\mu \in \mathbb{R}_{\geq 0}$. Therefore, defining the following special case of $f$-MIP will be useful.
\begin{definition}[$\mu$-Gaussian Membership Inference Privacy]
\label{def:mu-gmip}
Let $\Phi$ be the cumulative distribution function (CDF) of a standard normal distribution.
Define $g_\mu(\alpha)\coloneqq \Phi(\Phi^{-1}(1-\alpha)-\mu)$ to be the trade-off function derived from testing two Gaussians; one with mean $0$ and one with mean $\mu$. 
An algorithm $A$ is $\mu$-Gaussian Membership Inference private ($\mu$-GMIP) with privacy parameter $\mu$ if it is $g_\mu$-MIP, i.e., it is MI private with trade-off function $g_\mu$.
\end{definition}

\begin{remark}
DP can also be defined via the Gaussian trade-off function, which results in $\mu$-Gaussian Differential Privacy ($\mu$-GDP, \citep{dong2022gaussian}). While the trade-off curves for both $\mu$-GDP and $\mu$-GMIP have the same parametric form, they have different interpretations: $\mu$-GDP describes the trade-off function that an attacker with complete knowledge (left column in Table \ref{tab:dp_mip_comparison}) could achieve while $\mu$-GMIP describes the trade-off function that an attacker with MI attack capability can achieve (right column in Table \ref{tab:dp_mip_comparison}). 
In the next section we will quantify their connection further.
\end{remark}

\subsection{Relating $f$-MIP and $f$-DP}
We close this section by providing first results regarding the relation between $f$-DP and $f$-MIP. As expected, $f$-DP is strictly stronger than $f$-MIP, which can be condensed in the following result:
\begin{theorem}[$f$-DP implies $f$-MIP] \label{thm:mipweakerthandp}
Let an algorithm $A: D^{n} \rightarrow \mathbb{R}^d$ be $f$-differentially private \cite{dong2022gaussian}.
Then, algorithm $A$ will also be $f$-membership inference private.
\end{theorem}

We proof this result in \Cref{sec:proofmipweakerdp}.
This theorem suggests one intuitive, simple and yet actionable approach to guarantee Membership Inference Privacy.
This approach involves the use of DP learning algorithms such as DP-SGD \cite{abadi2016deep}, which train models using noised gradients.
However, as we will see in the next section, using noise levels to guarantee $f$-DP is usually suboptimal to guarantee $f$-MIP.

\section{Implementing $f$-MIP through Noisy SGD}
\label{sec:implementing_fmip}
We would now like to obtain a practical learning algorithm that comes with $f$-MIP guarantees.
As the dependency between the final model parameters and the input data is usually hard to characterize, we follow the common approach and trace the information flow from the data to the model parameters through the training process of stochastic gradient descent \citep{song2013stochastic,abadi2016deep}.
Since the gradient updates are the only path where information flows from the data into the model, it suffices to privatize this step.


\subsection{$f$-MIP for One Step of Noisy SGD}
We start by considering a single SGD step.
Following prior work \citep{song2013stochastic,abadi2016deep}, we make the standard assumption that only the mean over the individual gradients $\vm = \frac{1}{n}\sum_{i=1}^n \vtheta_i$ is used to update the model (or is published directly) where $\vtheta_i \in \mathbb{R}^d$ is a sample gradient.
Consistent with the definition of the membership inference game, the attacker tries to predict whether a specific gradient $\vtheta^\prime$ was part of the set $\left\{\vtheta_i\right\}_i$ that was used to compute the model's mean gradient $\vm$ or not.
We are interested in determining the shape of the attacker's trade-off function. 
For the sake of conciseness, we directly consider one step of noisy SGD (i.e., one averaging operation with additional noising, see \Cref{alg:dpsgd} from the Appendix), which subsumes a result for the case without noise by setting $\tau^2=0$.
We establish the following theorem using the Central Limit Theorem (CLT) for means of adequately large batches of $n$ sample gradients, which is proven in \Cref{sec:app_onestep_dp_sgd}.
\begin{theorem}[One-step noisy SGD is $f$-membership inference private]
Denote the cumulative distribution function of the non-central chi-squared distribution with $d$ degrees of freedom and non-centrality parameter $\gamma$ by $ F_{\chi^2_d(\gamma)}$. 
Let the gradients $\vtheta^\prime \in \mathbb{R}^d$ of the test points follow a distribution with mean $\vmu$ and covariance $\mSigma$, let $K \geq \lVert \mSigma^{-1/2} \vtheta^\prime \rVert_2^2$ and define $n_{\text{effective}} = n + \frac{\tau^2 n^2}{C^2}$.
For sufficiently large batch sizes $n$, one step of noisy SGD is $f$-membership inference private with trade-off function given by:
\begin{align}
\beta(\alpha) \approx 1-F_{\chi^2_d((n_{\text{effective}}-1)K)}\left(\frac{n_{\text{effective}}}{n_{\text{effective}}-1} F_{\chi^2_d\left(n_{\text{effective}}K\right)}^{-1}(\alpha) \right).
\end{align}
\label{theorem:onestep_dp_sgd}
\end{theorem}
\begin{wrapfigure}[13]{r}{0.33\textwidth}
\vspace{-0.8cm}
\centering
\includegraphics[scale=0.60]{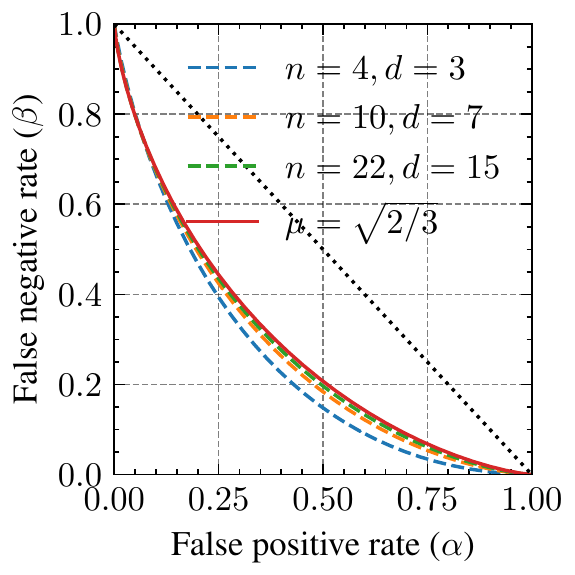}
\caption{\textbf{Trade-off function convergence.} The trade-off function from Theorem \ref{theorem:onestep_dp_sgd} converges to the one from Corollary~\ref{corollary:onestep_dp_sgd} where $\tau^2{=}0$ and $K{=}d$.}
\label{fig:vary_n_d}
\end{wrapfigure}
The larger the number of parameters $d$ and the batch size $n$ grow, the more the trade-off curve approaches the $\mu$-GMIP curve, which we show next (see Figure \ref{fig:vary_n_d}).
\begin{corollary}[One step noisy SGD is approx.\ $\mu$-GMIP]
For large $d,n$, noisy SGD is approximately $g_{\mu_{\text{Step}}}$-GMIP. In particular, $\beta(\alpha) \approx  \Phi(\Phi^{-1}(1-\alpha)-\mu_{\text{Step}})$ with privacy parameter:
\begin{align}
\mu_{\text{step}} = \frac{d + (2 n_{\text{effective}} - 1) K }{n_{\text{effective}} \sqrt{2d + 4 n_\text{effective} K}}.
\end{align} 
\label{corollary:onestep_dp_sgd}
\end{corollary} 
This result is striking in its generality as it also covers models trained without additional noise or gradient cropping ($n_{\text{effective}}=n$ in that case). Unlike for DP, even standard models trained with non-noisy SGD offer an elementary level of MIP. Our result further explicitly quantifies four factors that lead to attack success: the batch size $n$, the number of parameters $d$, the strength of the noise $\tau^2$ and the worst-case data-dependent \emph{gradient susceptibility} $\lVert \mSigma^{-1/2} \vtheta^\prime \rVert_2^2$.
The closeness of the trade-off function to the diagonal, which is equivalent to the attacker randomly guessing whether a gradient $\vtheta'$ was part of the training data or not, is majorly determined by the ratio of $d$ to $n$. 
The higher the value of $d$ relative to $n$, the easier it becomes for the attacker to identify training data points. 
Furthermore, a higher gradient susceptibility $K$, which measures the atypicality of a gradient with respect to the gradient distribution, increases the likelihood of MI attacks succeeding in identifying training data membership. 
It is worth noting that if we do not restrict the gradient distribution or its support, then there might always exist gradient samples that significantly distort the mean, revealing their membership in the training dataset.
This phenomenon is akin to the $\delta$ parameter in DP, which also allows exceptions for highly improbable events. 
\begin{remark}[Magnitude of $\mu_{\text{Step}}$]
When the dimensions of the uncorrelated components in $\mSigma^{-\frac{1}{2}}\vtheta'$ are also independent, we expect $K$ to follow a $\chi^2$-distribution with $d$ degrees of freedom and thus $K \in \mathcal{O}(d)$. 
In the standard SGD-regime ($\tau^2=0$) with $d, n \gg 1$, we obtain $\mu \in \mathcal{O}\big(\sqrt{{d}/{n}}\big)$.
\label{remark:convergence_mu_mip}
\end{remark}
\begin{remark}[On Optimality]
The dependency on $d$ when $\tau^2 > 0$ is a consequence of our intentionally broad proving strategy. 
Our proof approach consists of two key steps: First, we establish the optimal LRT under general gradient distributions, without adding noise or imposing any cropping constraints (See Appendix \ref{sec:app_onestep_dp_sgddatanoise}). This initial step serves as the foundation for our subsequent analysis and is (1) as general as possible covering all distributions with finite variance and is (2) optimal in the sense of the Neyman-Pearson Lemma, i.e., it cannot be improved. 
This means that our result covers all models trained with standard SGD ($\tau^2=0 \text{ and } C = \infty$) and is remarkable in its generality as it is the first to suggest clear conditions when adding noise is not required to reach $f$-MIP.
Second, we specialize our findings to cropped random variables with added noise (See Appendix \ref{sec:app_onestep_dp_sgdunitnoise}). This analysis could potentially be improved by considering individual gradient dimensions independently.
\end{remark}

\subsection{Composition and Subsampling}
\label{sec:composition}
In the previous section, we have derived the trade-off function for a single step of SGD. 
Since SGD is run over multiple rounds, we require an understanding of how the individual trade-off functions can be composed when a sequence of $f$-MIP operations is conducted, and a random subset of the entire data distribution is used as an input for the privatized algorithm.
The next lemma provides such a result for $\mu$-GMIP and follows from a result that holds for hypotheses tests between Gaussian random variables due to \citet{dong2022gaussian} (see \Cref{sec:app_compositionresults} for details and more results).


\begin{lemma}[Asymptotic convergence of infinite DP-SGD]
Let $n$ be the batch size in SGD, and $N$ be the entire size of the dataset. If a single SGD-Step is at least as hard as $\mu_{\text{step}}$-GMIP with respect to the samples that were part of the batch and $\frac{n\sqrt{t}}{N} \rightarrow c$ as $\lim_{t\rightarrow \infty}$ (the batch size is gradually decreased), then the noisy SGD algorithm will be $\mu$-GMIP with \label{lem:composition} 
\begin{align}
    \mu = \sqrt{2} c\sqrt{\exp(\mu_{\text{step}}^2) \Phi\left(1.5 \mu_{\text{step}}\right) + 3\Phi\left(-0.5 \mu_{\text{step}}\right)-2}.\label{eqn:dongfactor}
\end{align}
\vspace{-0.35cm}
\end{lemma}
Note that this result also provides a (loose) bound for the case where exactly $T$ iterations are run with a batch size of $n^\prime$ with $c=\frac{n^\prime\sqrt{T}}{N}$ (through using $n(t)=n^\prime~\text{if}~t\leq T,~\text{else}~n(t) = \frac{n^\prime\sqrt{T}}{\sqrt{t}}$).
With this result in place, we can defend against MI attacks using the standard noisy SGD algorithm.

%% file: 6-experiments.tex
\textbf{Datasets and Models.} We use three datasets that were previously used in works on privacy risks of ML models \cite{shorki2021explanation}: The CIFAR-10 dataset which consists of 60k small images \citep{krizhevsky2009learning}, the Purchase tabular classification dataset \citep{nasr2018MembershipInferenceAdvers} and the Adult income classification dataset from the UCI machine learning repository \citep{Dua2019uci}. 
Following prior work by \citet{abadi2016deep}, we use a model pretrained on CIFAR-100 and finetune the last layer on CIFAR-10 using a ResNet-56 model for this task \citep{he2016deep} where the number of fine-tuned parameters equals $d=650$. 
We follow a similar strategy on the Purchase dataset, where we use a three-layer neural network. 
For finetuning, we use the 20 most common classes and $d=2580$ parameters while the model is pretrained on 80 classes.
On the adult dataset, we use a two-layer network with 512 random features in the first layer trained from scratch on the dataset such that $d=1026$.
We refer to \Cref{sec:app_modeltraining} for additional training details.  We release our code online.\footnote{\url{https://github.com/tleemann/gaussian_mip}}

\subsection{Gradient Attacks Based on the Analytical LRT}

\begin{wrapfigure}[19]{r}{9cm}
\vspace{-0.85em}
\small\rule{\linewidth}{0.4pt}\vspace{-0.7em}
\captionof{algorithm}{Gradient Likelihood Ratio (GLiR) Attack\label{alg:glir}}\vspace{-0.7em}
\small\rule{\linewidth}{0.4pt}
\begin{algorithmic}[1]
\Require Training data distribution $\mathcal{D}$, batch size $n$, number of parameters $d$, query point $\bm{x} \in D$, averaged gradients of each batch $\vm_t \in \mathbb{R}^{d}$ for training steps $t=1,\ldots, T$, parameter gradient computation function $\nabla_{\vw_t} \mathcal{L}: D \rightarrow \mathbb{R}^{d}$ of training, threshold $\eta$
\State $p_{\text{total}} \leftarrow 0$
\For{$t = 1,\dots,T$}
    \State $B=\left\{\vb_1, \ldots, \vb_m\right\}\sim D^m$ \Comment  Sample background data
    \State $\vg_i = \nabla_{\vw}\mathcal{L}(\vb_i), i=1...m$ \Comment Compute background gradients
    \State $\hat{\bm{\Sigma}} = \text{Cov}\left\{\vg_1, ..., \vg_m \right\} \in \mathbb{R}^{d \times d}$ \Comment Approximate covariance $\bm{\Sigma}$
    \State $\hat{\bm{\mu}} = \text{Mean}\left\{\vg_1, ..., \vg_m \right\} \in \mathbb{R}^{d}$ \Comment Approximate mean $\bm{\mu}$
    \State $\bm{\theta} = \nabla_{w_t}\mathcal{L}(\vx)$ \Comment Compute  gradients for the query point 
    \State  $\hat{S}=(n{-}1)\left(\vm_t{-}\bm{\theta}\right)^\top\hat{\bm{\Sigma}}^{-1}\left(\vm_t{-}\bm{\theta}\right)$ \Comment Compute test statistic
    \State  $\hat{K} =  \lVert \hat{\mSigma}^{-1/2} (\vtheta-\hat{\vmu})\rVert_2^2$ \Comment Estimate gradient susceptibility
    \State $p_{\text{step}}=\log \text{F}^{-1}_{\chi^2_d( n\hat{K})}(\hat{S})$ \Comment Compute $\log p$-value under $H_0$
    \State $p_{\text{total}} \leftarrow p_{\text{total}} +p_{\text{step}}$
\EndFor\\
\Return \texttt{Train} \textbf{if} $p_{\text{total}} < \eta$, \textbf{else} \texttt{Test}
\label{alg:gradientattack}
\end{algorithmic}
\vspace{-0.7em}
\rule{\linewidth}{0.4pt}
\end{wrapfigure} 

To confirm our theoretical analysis for one step of SGD and its composition, we implement the gradient attack based on the likelihood ratio test derived in the proof of \Cref{theorem:onestep_dp_sgd}. We provide a sketch of the implementation in \Cref{alg:glir} and additional details in \Cref{sec:app_gradientlrtattack}.
An essential requirement in the construction of the empirical test is the estimation of the true gradient mean $\vmu$ and the true inverse covariance matrix $\mSigma^{-1}$ since these quantities are essential parts of both the test statistic $\hat{S}$ and the true gradient susceptibility term $\hat{K}$
needed for the analytical attack.
The attacker uses their access to the gradient distribution (which is standard for membership inference attacks \citep{carlini2021membership,pawelczyk2022privacy} and realistic in federated learning scenarios \citep{kairouz2021advances}), to estimate the distribution parameters. 
In practice, however, the empirical estimates of $\hat{\vmu}$, $\hat{\mSigma}^{-1}$ and thus $\hat{K}$ will be noisy and therefore we do not expect that the empirical trade-off curves match the analytical curves exactly.

Using our novel Gradient Likelihood Ratio (GLiR) attack we can audit our derived guarantees and their utility.
First, we audit our one-step guarantees from \Cref{{theorem:onestep_dp_sgd}}. To compare the models, we adapt the batch size $n$ such that all models reach the same level of $\mu$-GMIP.
In \Cref{fig:gradient1}, we use a simulated gradient distribution with known parameters $\vmu, \mSigma^{-1}$ and $d$. 
In this case, we can estimate $K$ accurately and observe that our bounds are tight when the distribution parameters and thus the respective gradient susceptibilities can be computed accurately.
We provide additional ablation studies that gauge the approximation quality of  with small values for $d$, $n$ and different simulated distributions in \Cref{sec:app_ablations}.
When the parameters are unknown and we have to estimate the parameters, our attacks become weaker and do not match the analytical prediction (see \Cref{fig:gradient2}). 

We also audit our composition guarantees.
We do five SGD-steps in \Cref{fig:gradient3}. 
While there is a small gain in attack performance on the CIFAR-10 dataset (e.g., at FPR=0.25), the attack performance on the other datasets remains largely unaffected.
This mismatch occurs since the theoretical analysis is based on the premise that the attacker gains access to independently sampled gradient means for each step to separate training and non-training points, but in practice we do not gain much new information as the model updates are not statistically independent and too incremental to change the gradient means significantly between two subsequent steps.
Therefore, a practical attacker does not gain much additional information through performing several steps instead of one. 
Future work is required to model these dependencies and potentially arrive at a tighter composition result under incremental parameter updates. We provide results for additional existing membership inference attacks, for instance the recent loss-based likelihood-ratio attack by \citet{carlini2021membership} in \Cref{sec:app_additionalattacks}, which all show weaker success rates than the gradient-based attack that proved most powerful in our setting.

\begin{figure}[t!]
\centering
\begin{subfigure}{0.32\textwidth}
\centering
\includegraphics[width=\linewidth]{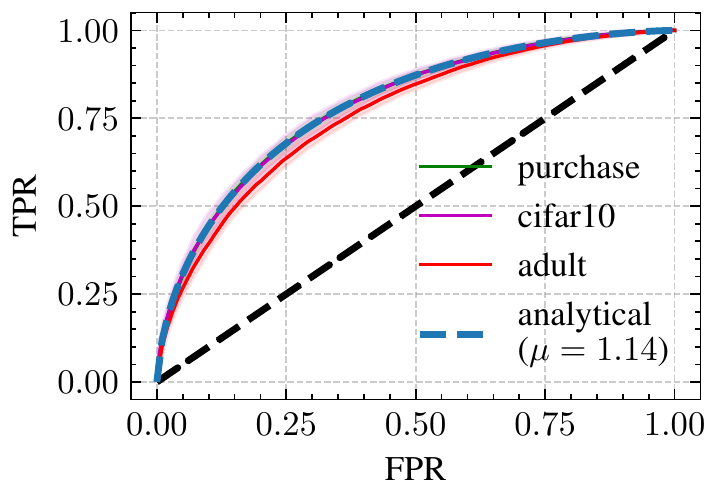}
\caption{Single step of simulated gradient distribution with known parameters.}
\label{fig:gradient1}
\end{subfigure}
\hfill
\begin{subfigure}{0.32\textwidth}
\centering
\includegraphics[width=\linewidth]{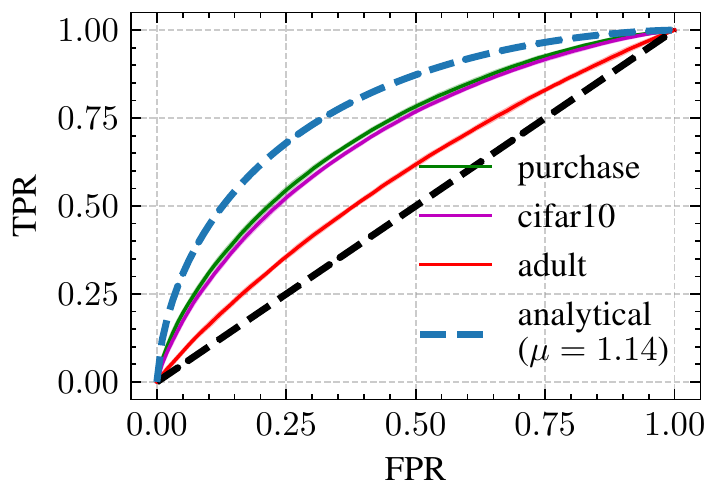}
\caption{Single step with real model gradients and estimated parameters.}
\label{fig:gradient2}
\end{subfigure}
\hfill
\begin{subfigure}{0.32\textwidth}
\centering
\includegraphics[width=\linewidth]{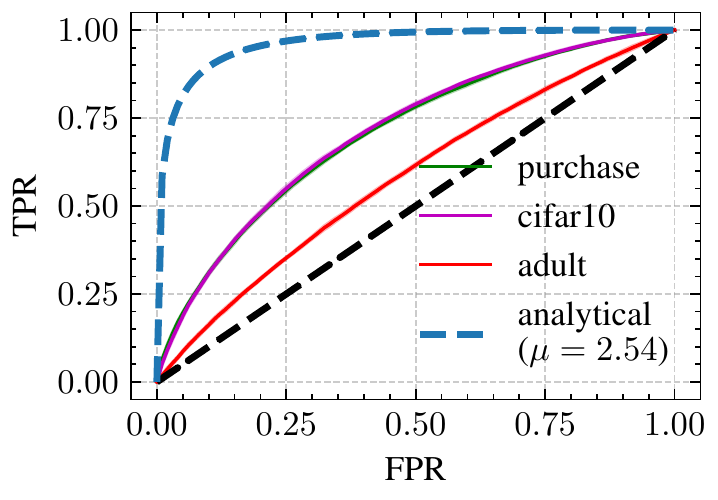}
\caption{As in (b), but now composition of 5 steps for real model gradients.}
\label{fig:gradient3}
\end{subfigure}
\caption{\textbf{Auditing $f$-MIP with our gradient attack (GLiR) when $\tau^2=0$}. We show trade-off curves when the gradient distribution is known (a) and when the gradients are obtained from a trained model that was finetuned on various data sets (b, c). The analytical solutions are computed with a value of $K=d$ and using the composition result for $k$ steps in \Cref{sec:app_compositionresults} for (c).}
\vspace{-0.25cm}
\label{fig:verification_experiment}
\end{figure}

\subsection{Comparing Model Utility under $\mu$-GDP and $\mu$-GMIP}
Here we compare the utility under our privacy notion to the utility under differential privacy. 
We sample 20 different privacy levels ranging from $\mu\in [0.4, ... ,50]$ and calibrate the noise in the SGD iteration to reach the desired value of $\mu$. 
We can do so both for $\mu$-GMIP using the result in \Cref{eqn:dongfactor} and using the result by \citet[Corollary 4]{dong2022gaussian} for $\mu$-GDP, which result in the same attack success rates while $\mu$-GDP allows for stronger privacy threat models. Due to \Cref{thm:mipweakerthandp}, we never need to add more noise for $\mu$-GMIP than for $\mu$-DP. Further details are provided in \Cref{sec:app_modeltraining}.
\Cref{fig:utility} shows a comparison of the accuracy that the models obtain. 
We observe that the model under GMIP results in significantly higher accuracy for most values of $\mu$. 
As $\mu\rightarrow 0$ both privacy notions require excessive amounts of noise such that the utility decreases towards the random guessing accuracy.
On the other hand, for higher values of $\mu$, there is no need to add any noise to the gradient to obtain $\mu$-GMIP, allowing to obtain the full utility of the unconstrained model.
This indicates that useful GMIP-bounds do not necessarily require noise. 
For instance, on the CIFAR-10 model, no noise is required for $\mu \geq 0.86$ which is a reasonable privacy level \citep{dong2022gaussian}. 
Overall, these results highlight that useful and interpretable privacy guarantees can often be obtained without sacrificing utility.

\begin{figure}[tb!]
\centering
\begin{subfigure}{0.32\textwidth}
\centering
\includegraphics[width=\textwidth]{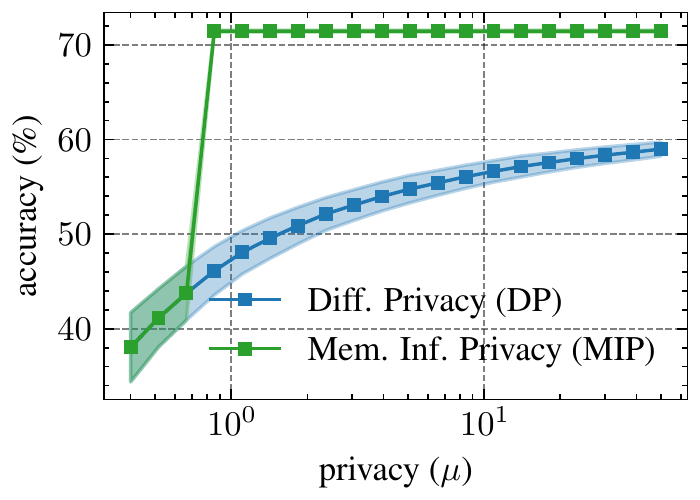}
\caption{CIFAR-10}
\label{fig:util1}
\end{subfigure}
\hfill
\begin{subfigure}{0.32\textwidth}
\centering
\includegraphics[width=\textwidth]{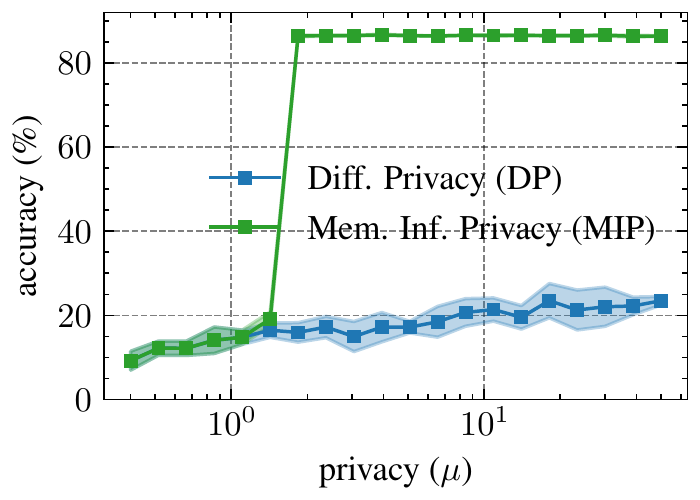}
\caption{Purchase}
\label{fig:util2}
\end{subfigure}
\hfill
\begin{subfigure}{0.32\textwidth}
\centering
\includegraphics[width=\textwidth]{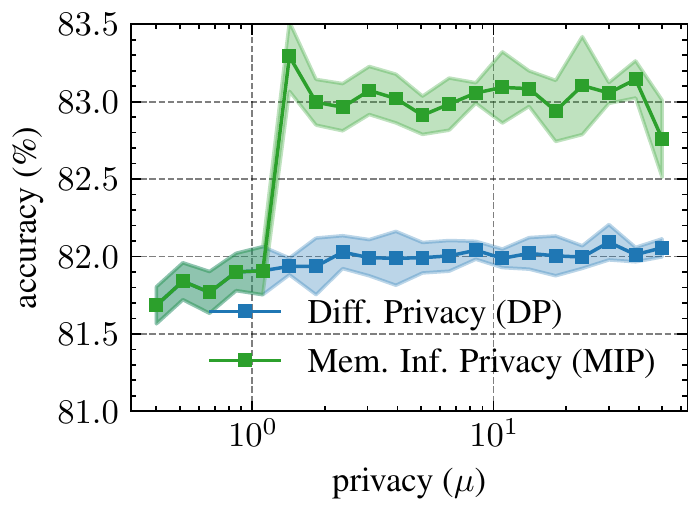}
\caption{Adult}
\label{fig:util3}
\end{subfigure}
\caption{\textbf{Utility of DP versus MIP.} Model performance on three datasets across different privacy levels $\mu$ (small $\mu$ denotes high privacy) using the notions of $\mu$-Gaussian Differential Privacy (parametric form of $f$-DP, \cite{dong2022gaussian}) and $\mu$-Gaussian Membership Inference Privacy (parametric form of $f$-MIP, ours) on three datasets. GMIP usually allows for substantially increased accuracy over the corresponding GDP guarantee with the same attack success rates controlled by $\mu$. However, the attacker under GMIP runs membership inference (MI) attacks while GDP allows for a wider set of privacy threat models. For more details on differences in the underlying threat models see \Cref{tab:dp_mip_comparison}.
\label{fig:utility}}
\vspace{-0.62cm}
\end{figure}

%% file: 7-conclusion.tex
In the present work, we derived the general notion of $f$-Membership Inference Privacy ($f$-MIP) by taking a hypothesis testing perspective on membership inference attacks.
We then studied the noisy SGD algorithm as a model-agnostic tool to implement $f$-Membership Inference Privacy, while maintaining Differential Privacy (DP) as a worst-case guarantee. 
Our analysis revealed that significantly less noise may be required to obtain $f$-MIP compared to DP resulting in increased utility.
Future work is required to better model the dependencies when composing subsequent SGD steps which could lead to improved bounds in practice. 
Furthermore, our analysis shows that when the capacity of the attacker is further restricted, e.g., to API access of predictions, there remains a gap between our theoretical bounds and loss-based membership inference attacks that can be implemented for real models. 
More work is required to either produce more sophisticated attacks or derive theoretical bounds for even less powerful attackers to close this gap.

%% file: 8-appendix.tex
\section{Algorithms}
\label{app:algorithm}
\textbf{Reviewing Noisy SGD.}
Noisy SGD, also known as DP-SGD when appropriately parameterized, is the most prevalent algorithm to train differentiable machine learning models subject to DP privacy constraints. 
In the main text, we have shown that this algorithm, when appropriately parameterized, can be used to train $f$-MIP models, too.
Since our gradient attack relies on the inner workings of the algorithm, we review it here for the reader's convenience.
DP-SGD works by clipping the individual gradients in each batch, taking the mean over the these gradients in a batch, and finally adding noise of magnitude $\tau$ to them.  This process is then iterated over $T$ epochs. Pseudo code is shown in Algorithm \ref{alg:dpsgd}.

We note that the parametrization of the noise level is different across recent works. While $\tau$ corresponds to the noise added to the entire batch, other works such as \cite{abadi2016deep, dong2022gaussian} use different parameters to characterize the noise level. For instance, \citet{dong2022gaussian} add noise of magnitude $\tau^2 = \frac{4\sigma^2C^2}{n^2}$ and use $\sigma$ to characterize the noise level.

\begin{algorithm}[htb]
\caption{Noisy Stochastic Gradient Descent (Noisy SGD)}
\begin{algorithmic}[1]
\Require Training data $D = \{(\vx_i,y_i)\}_{i=1}^N$, loss function $\mathcal{L}$, learning rate $\eta$, batch size $n$, number of iterations $T$, gradient norm bound $C \in \mathbb{R}_+$, noise scale $\tau \in \mathbb{R}_+$
\State Initialize model parameters $\vtheta$ randomly
\For{$t = 1,\dots,T$}
    \State Sample a batch $B_t$ of size $n$ uniformly at random from $D$
    \State \textbf{Compute gradients}
    \State For each $(\vx_i, y_i) \in B_t$ compute $\vg(\vx_i, y_i) =\nabla \mathcal{L}(\vtheta,\vx_i, y_i)$
    \State \textbf{Clip gradients} (to have norm at most $C$)
    \State $\vg(\vx_i, y_i) \gets \vg(\vx_i, y_i)\cdot\max\left(1,\frac{C}{\lVert\vg(\vx_i, y_i)\rVert}\right), (\vx_i, y_i) \in B_t$ 
    \State \textbf{Aggregate and noise gradients}
    \State $\tilde{\vg} \gets \left(\frac{1}{n} \sum_i \vg(\vx_i, y_i) \right) + \mathcal{N}\left(\mathbf{0}, \tau^2\mI\right)$
    \State \textbf{Update parameters} 
    \State $\vtheta \gets \vtheta - \eta \tilde{\vg}$
\EndFor
\State Return $\vtheta$
\end{algorithmic}
\label{alg:dpsgd}
\end{algorithm}

\section{Additional Information on Related Work}
\label{app:add_details_related_work}

\textbf{Privacy Auditing.} 
Our research is also linked to the literature focusing on the validation of theoretical privacy guarantees, also known as privacy auditing. 
This literature involves the assessment of privacy breaches in private algorithms \citep{jagielski2020auditing}.
Typically, privacy auditing entails the utilization of membership inference attacks \citep{yeom2018privacy,shokri2017membership}, wherein the effectiveness of the attack is translated into an empirical approximation of the privacy level, denoted as $\hat{\varepsilon}$.
While most existing privacy auditing methods, such as those by \citet{jagielski2020auditing}, \citet{nasr2021adversary, nasr2023tight} and \citet{zanella2023bayesian} are computationally extensive as they require multiple shadow models to be fitted to conduct privacy audits, more recent works suggest privacy audits based on a single model fit \citep{steinke2023privacy,andrew2023one,maddock2022canife}.
Similar to this recent line of work, our privacy notion can be audited by a single model fit. 
However, our work differs in the sense that our auditing algorithm precisely evaluates our suggested privacy notion $f$-MIP.

\textbf{Extended comparison to privacy attacks.}
There is a long line of prior work developing \citep{shokri2017membership,long2018understanding,sablayrolles2019whitebox,carlini2021membership,haim2022reconstructing,carlini2023extracting,pawelczyk2022privacy} or analyzing \citep{thudi2022bounding,tan2022parameters,tan2023blessing} privacy attacks on machine learning models. 
A common class of attacks called \emph{membership inference attacks} focus on determining if a given instance is present in the training data of a particular model \citep{shokri2017membership,shorki2021explanation,pawelczyk2022privacy,Choquette2019label,chen2020gan_survey,sablayrolles2019whitebox,yeom2018privacy,carlini2021membership,ye2021enhanced}.
Most of these attacks typically exploit the differences in the distribution of model confidence on the true label (or the loss) between the instances that are in the training set and those that are not \citep{shokri2017membership, sablayrolles2019whitebox, carlini2021membership, ye2021enhanced}.
For example, \citet{shokri2017membership} proposed a loss-based membership inference attack which determines if an instance is in the training set by testing if the loss of the model for that instance is less than a specific threshold. 
Other membership inference attacks are also predominantly loss-based attacks where the calibration of the threshold varies from one proposed attack to the other \citep{sablayrolles2019whitebox, carlini2021membership, ye2021enhanced}. 
Some works leverage different information that goes beyond the loss functions to do membership inference attacks.
For instance, \citet{shorki2021explanation} and \citet{pawelczyk2022privacy} leverage model explanations to orchestrate membership inference attacks.

\textbf{Comparison to existing attacks.}
In Table \ref{tab:summary_assumptions}, we summarize the assumptions underlying different membership inference attacks.
Note that our attack does not require the training of multiple shadow models on data from the data distribution $\mathcal{D}^N$. 
Instead, we derive the distributions of the LRT test statistic under the null and alternative hypotheses in closed form (see \Cref{sec:app_onestep_dp_sgd}), which drops the requirement of training (appropriately parameterized) shadow models to approximate these two distributions.
These shadow models can be trained since the attacker is allowed access to the general data distribution $\mathcal{D}$. 
Similar to other LRT attacks, our attack also requires access to $\mathcal{D}$ to approximate the parameters $\mSigma, \vmu$ and $K$ required for the construction and verification of our likelhood ratio based attack.
As opposed to other attacks, our attack is based on the requirement that the attacker has access to model gradients which is a realistic assumption in many federated learning scenarios \citep{kairouz2021advances}.
Appendix \ref{sec:app_gradientlrtattack} summarizes our gradient based LRT attack in more detail.
\begin{table}[tb]
\centering
\resizebox{\columnwidth}{!}{%
\begin{tabular}{cccccc}
\toprule
 Info & \texttt{Loss} \citep{yeom2018privacy} & \texttt{CFD} \citep{pawelczyk2022privacy} & \texttt{Loss LRT} \citep{carlini2021membership} & \texttt{CFD LRT} \citep{pawelczyk2022privacy} & \texttt{Gradient LRT} \\
 \cmidrule(lr){1-1}  \cmidrule(lr){2-6}
Query access to $f_{\btheta}$ & $\checkmark$ & $\times$ & $\checkmark$ & $\times$ & $\times$  \\
Query access to $\nabla f_{\btheta}$ & $\times$ &  $\times$  &  $\times$ &  $\times$  &  $\checkmark$   \\
Query access to $\mathcal{R}$ & $\times$ &  $\checkmark$  &  $\times$ &  $\checkmark$  &  $\times$   \\
Known loss function & $\checkmark$ & $\times$ & $\checkmark$ & $\times$  & $\times$ \\
Access to $\D^N$ & $\times$ & $\times$ & $\checkmark$ & $\checkmark$ & $\checkmark$ \\
Access to true labels & $\checkmark$ & $\times$ & $\checkmark$ & $\times$ & $\times$ \\
Analytical & $\times$ & $\times$ & $\times$ & $\times$ & $\checkmark$ \\
Shadow models & $\times$ & $\times$ & $\checkmark$ & $\checkmark$ & $\times$ \\
\bottomrule
\end{tabular}
}
\caption{Summarizing the assumptions underlying the different MI attacks. The recourse based attacks do not require access to the true labels nor do they need to know the correct loss functions, but they additionally require access to a recourse generating API $\mathcal{R}$. To the best our knowledge, our gradient attack is the only one for which analytical results exist.}
\label{tab:summary_assumptions}
\end{table}

\begin{table}[htb]
\centering
\resizebox{\columnwidth}{!}{%
\begin{tabular}{cccccccccc}
\toprule
Type & Dataset & \# Samples (N) & \# Parameters (d) & Batch size (n) & Epochs & C & $\tau^2$ & Architecture  \\
\midrule
 I & CIFAR-10 & 500 & 650 & 500 & 5 & 10.0 & 0.0 & ResNet56 \\
 T & Purchase & 1970 & 2580 & 1970 & 5 & 10.0 & 0.0 & 3 layer DNN  \\
 T & Adult & 790 & 1026 & 790 & 5 & 10.0 & 0.0 & Random feature NN \\
\bottomrule
\end{tabular}
}
\caption{The parameters for the verification experiment are chosen so that the analytical privacy levels from Figure \ref{fig:verification_experiment} are $\mu_{\text{step}}=1.13$ and $\mu=2.54$, respectively. Note that ``I'' denotes image and ``T'' denotes tabular.}
\label{tab:verification_hyperparams}
\end{table}
\begin{table}[htb]
\centering
\resizebox{\columnwidth}{!}{%
\begin{tabular}{ccccccccccc}
\toprule
Type & Dataset & \# Samples (N) & \# Parameters (d, =K) & Batch size (n) & Epochs & C & Architecture  \\
\midrule
 I & CIFAR-10 & 48000 & 650 & 400 & 10 &  500.0  & ResNet56 \\
 T & Purchase & 54855 & 2580 & 795 & 3 &  2000.0 & 3 layer DNN \\
 T & Adult & 43000 & 1026 & 1000 & 20 &  800.0  & Random feature NN \\
\bottomrule
\end{tabular}
}
\caption{Parameters for the utility experiment from Figure \ref{fig:utility}.  Note that ``I'' denotes image and ``T'' denotes tabular. The dataset size were chosen to make them divisible by the batch size. The required noise $\tau^2$ is determined by the required privacy level $\mu$.}
\label{tab:utility_hyperparams}
\end{table}

\textbf{Comparing the threat models underlying f-DP and f-MIP.}
We note that the underlying threat model in membership inference (MI) attacks features several key differences to the threat model underlying DP, which controls an attacker's capacity to distinguish \emph{any} two neighboring datasets $D$ and $D^\prime$. 

First, in MI attacks, the datasets are sampled from the distribution $\mathcal{D}$, whereas DP protects all datasets which corresponds to granting the attacker the capacity of full dataset manipulation. 
Thereby the MI attack model is sensible in cases where the attacker cannot manipulate the dataset through injection of malicious samples (``canaries''). 
Instead, the notion of MI attacks is more realistic in cases when an attacker only has API access or access to the trained model but cannot interfere during training. 
Typically membership inference attacks are a fundamental ingredient in crafting data extraction attacks \citep{carlini2021extracting}, and hence we expect a privacy notion based on the membership inference threat model to be broadly applicable. 
Second, the samples that are protected under this notion are also drawn from the distribution. Consequently, MI primarily protects typical samples. 
In most cases, the distribution covers the data that the model is conceived to handle in practice, such that protecting against extreme outliers may be overconstraining. 
Finally, the goals of the attackers in both threat models are different.
Instead of being able to tell apart two datasets, the MI attacker is interested in inferring whether a given sample was part of the model's training set. 
As the sample is already known and only a binary response is required, this goal is weaker than other types of attacks such as full reconstruction attacks \citep{carlini2021extracting, haim2022reconstructing}. 
Therefore, the MI threat model covers many goals of realistic attackers.
For the reader's convenience, we replicate the tabular overview over these key differences from \Cref{tab:dp_mip_comparison} below.


\section{Experimental Details}
\subsection{Hyperparameters}
\label{sec:app_modeltraining}
In this section, we summarize the parameter settings for our experiments.
In \Cref{tab:verification_hyperparams}, we provide details on the verification experiment shown in Figure \ref{fig:verification_experiment}.
In Table \ref{tab:utility_hyperparams}, we summarize the hyperparameters used in our utility experiment shown in Figure \ref{fig:utility}.

To compute the analytical privacy levels in \Cref{fig:verification_experiment}, we use \Cref{corollary:onestep_dp_sgd} with $K=d$ and $\tau=0$, resulting in $\mu_{\text{step}}=\sqrt{\frac{2d}{2n+1}} \approx 1.14$ with the batch sizes and model parameters in \Cref{tab:verification_hyperparams}. For the last Figure, we use the result shown in \Cref{sec:app_compositionresults}, indicating that the combined privacy level when performing $k$ steps of SGD without subsampling and individual privacy level $\mu_{\text{step}}$, is given by $\mu=\sqrt{k}\mu_{\text{step}} =\sqrt{5}\cdot \mu_{\text{step}}  \approx 2.54$ (calculation was performed prior to rounding).

For the utility experiment, we chose the noise level $\tau$ according to \Cref{eqn:dongfactor} when we use MIP. For Differential Privacy, we use the result by \citet[Corollary 4]{dong2022gaussian}. However our $\tau$ has the following relation to the $\sigma$ by Dong et al., $\tau = \frac{2C}{n}\sigma$ so that we plug in $\sigma= \frac{n}{2C}\tau$ in Corollary 4 of \citet{dong2022gaussian}. We solve both \Cref{eqn:dongfactor} and Corollary 4 of \cite{dong2022gaussian} numerically to obtain the level $\tau_{\text{MIP}}$ required to obtain $\mu$-GMIP and $\tau_{\text{DP}}$ for $\mu$-GDP for 20 values of $\mu$ between 0.4 and 50 that are linearly spaced in logspace. Due to \Cref{thm:mipweakerthandp}, we never need to add more noise for $\mu$-GMIP than for $\mu$-GDP. Therefore, we set $\tau_{\text{MIP}} \leftarrow \min\left\{\tau_{\text{MIP}}, \tau_{\text{DP}}\right\}$, i.e., we take the minimum of the noise levels required for GDP, GMIP when we would like to guarantee GMIP. We obtain the values given in \Cref{tab:tauvalues}.

\begin{table}[tb]
\centering
\resizebox{\columnwidth}{!}{%
\begin{tabular}{r*{20}{c}}
\toprule
$\mu=$ & 0.40 & 0.52 & 0.66 & 0.86 & 1.11 & 1.43 & 1.84 & 2.37 & 3.05 & 3.94 & 5.08 & 6.55 & 8.44 & 10.88 & 14.03 & 18.09 & 23.33 & 30.08 & 38.78 & 50.00\\
\midrule
CIFAR-10 (MIP) & 2.84 & 2.44 & 2.13 & 0.00 & 0.00 & 0.00 & 0.00 & 0.00 & 0.00 & 0.00 & 0.00 & 0.00 & 0.00 & 0.00 & 0.00 & 0.00 & 0.00 & 0.00 & 0.00 & 0.00 \\
CIFAR-10 (DP) & 2.84 & 2.44 & 2.13 & 1.89 & 1.70 & 1.55 & 1.42 & 1.32 & 1.24 & 1.17 & 1.11 & 1.06 & 1.02 & 0.98 & 0.94 & 0.91 & 0.88 & 0.85 & 0.83 & 0.81 \\
Purchase (MIP) & 4.72 & 4.14 & 3.68 & 3.32 & 3.04 & 2.81 & 0.00 & 0.00 & 0.00 & 0.00 & 0.00 & 0.00 & 0.00 & 0.00 & 0.00 & 0.00 & 0.00 & 0.00 & 0.00 & 0.00 \\
Purchase (DP) & 4.72 & 4.14 & 3.68 & 3.32 & 3.04 & \emph{2.81} & 2.62 & 2.46 & 2.32 & 2.21 & 2.11 & 2.02 & 1.94 & 1.87 & 1.81 & 1.75 & 1.70 & 1.65 & 1.61 & 1.57 \\
Adult (MIP) & 3.38 & 2.77 & 2.30 & 1.93 & 1.65 & 0.00 & 0.00 & 0.00 & 0.00 & 0.00 & 0.00 & 0.00 & 0.00 & 0.00 & 0.00 & 0.00 & 0.00 & 0.00 & 0.00 & 0.00 \\
Adult (DP) & 3.38 & 2.77 & 2.30 & 1.93 & 1.65 & 1.43 & 1.26 & 1.13 & 1.02 & 0.94 & 0.87 & 0.81 & 0.77 & 0.73 & 0.69 & 0.66 & 0.63 & 0.61 & 0.59 & 0.57 \\
\bottomrule
\end{tabular}
}
\caption{Values of $\tau$ obtained for the Utility Experiment. We observe that for the higher values of $\mu$ there is often no need to add any noise to the gradients to obtain MIP, whereas substantial noise still needs to be added in the case of DP, which results in the reduced utility observed in \Cref{fig:utility}.\label{tab:tauvalues}}
\end{table}

\subsection{Ablation studies}
\label{sec:app_ablations}
We show results of several ablation studies in \Cref{fig:ablations}. The correspond to the verification setup used in \Cref{fig:gradient1} with simulated gradients and confirm that the approximations made in our analysis even hold for very small values of $d=2, n=5$ or $C=1$ and regardless of the gradients' distribution. We provide log-log scale plots for the setups corresponding to \Cref{fig:verification_experiment} in \Cref{fig:verification_experimentlogscale}.

\newcommand{\capvsp}{\vspace{-0.3cm}}
\begin{figure}[tb]
\centering
\begin{subfigure}{0.24\textwidth}
\centering
\includegraphics[width=\textwidth]{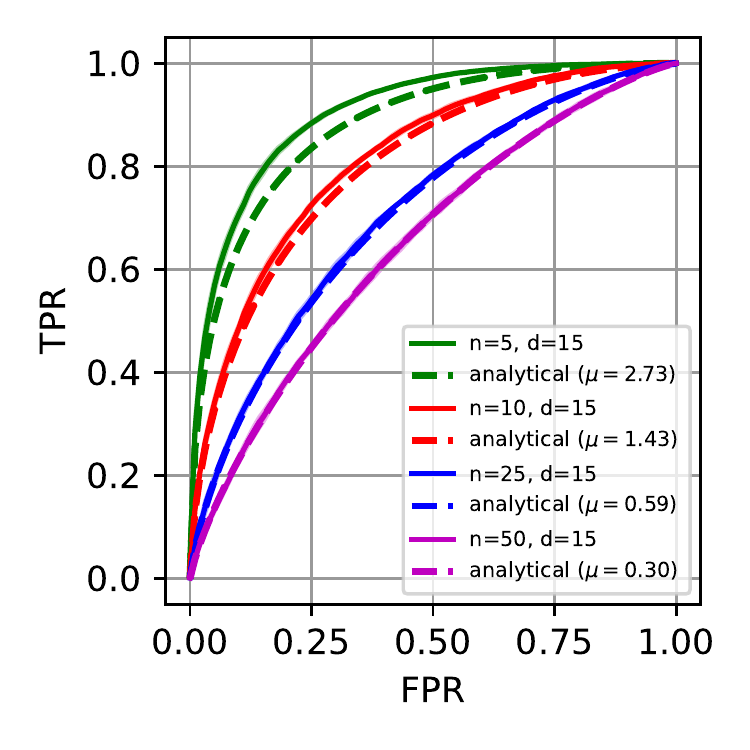}\capvsp
\caption{Changing $n$}
\label{fig:util1appa}
\end{subfigure}
\begin{subfigure}{0.24\textwidth}
\centering
\includegraphics[width=\textwidth]{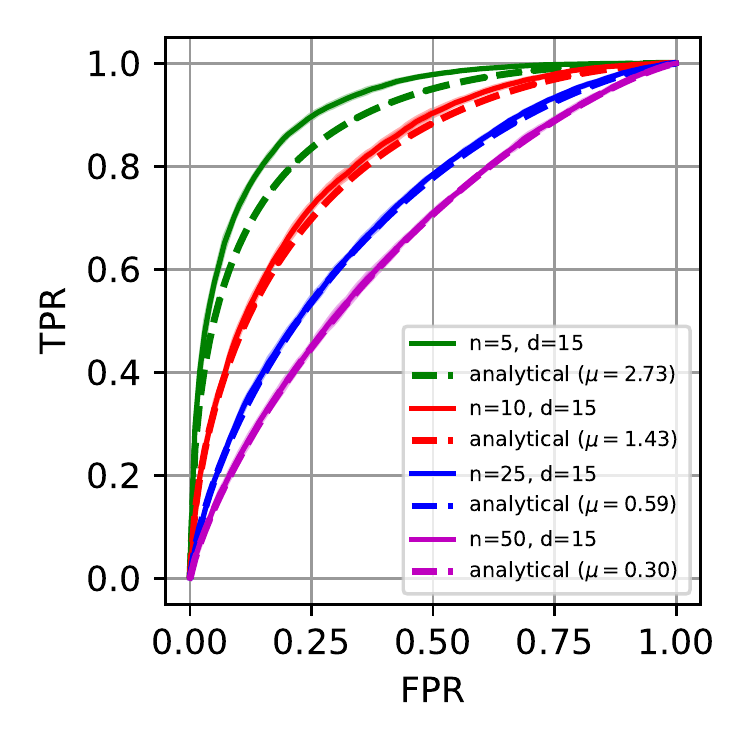}\capvsp
\caption{Changing $n$ (unif. dist.)}
\label{fig:util1appb}
\end{subfigure}
\begin{subfigure}{0.24\textwidth}
\centering
\includegraphics[width=\textwidth]{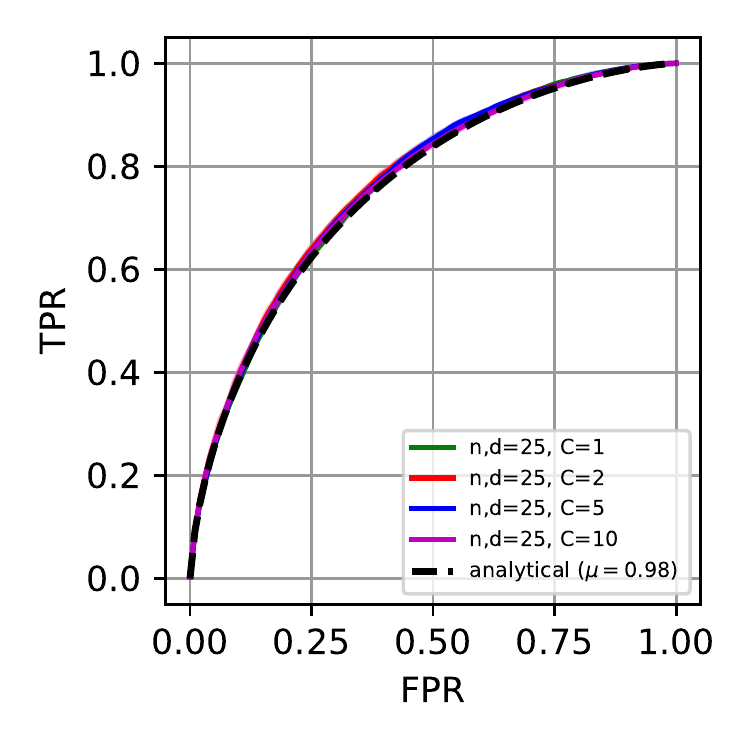}\capvsp
\caption{Changing $C$}
\label{fig:util2appc}
\end{subfigure}
\begin{subfigure}{0.24\textwidth}
\centering
\includegraphics[width=\textwidth]{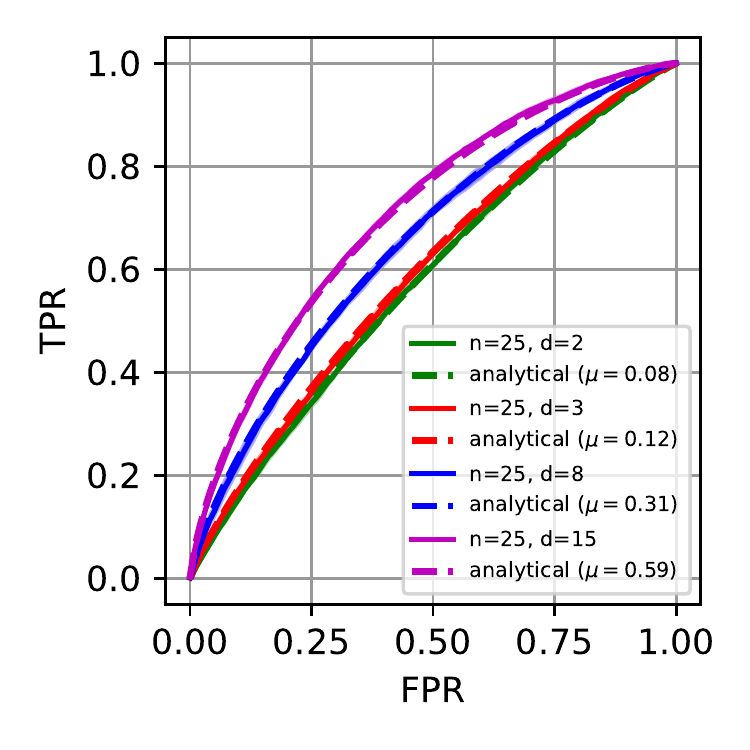}\capvsp
\caption{Changing $d$}
\label{fig:util3appd}
\end{subfigure}
\begin{subfigure}{0.24\textwidth}
\centering
\includegraphics[width=\textwidth]{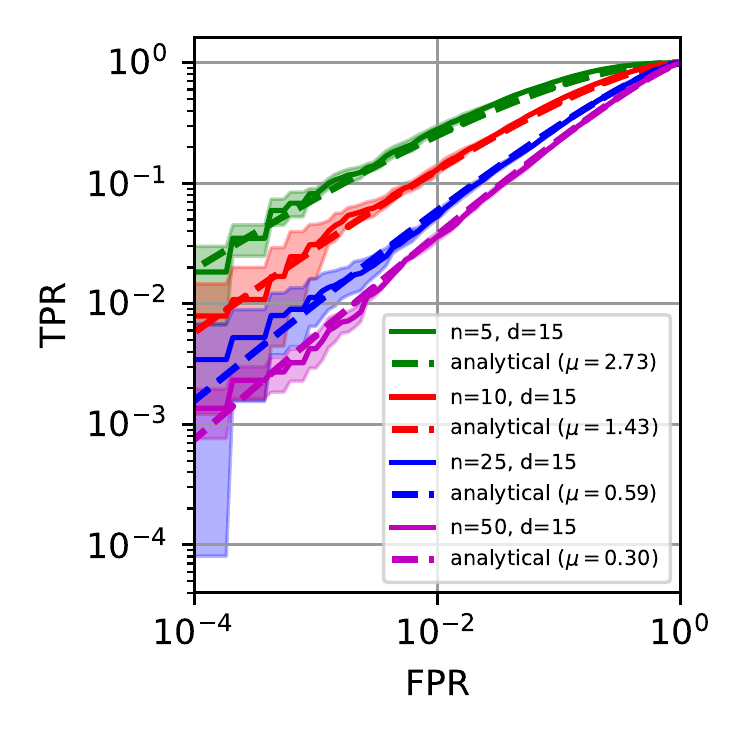}\capvsp
\caption{Changing $n$}
\label{fig:util1appe}
\end{subfigure}
\begin{subfigure}{0.24\textwidth}
\centering
\includegraphics[width=\textwidth]{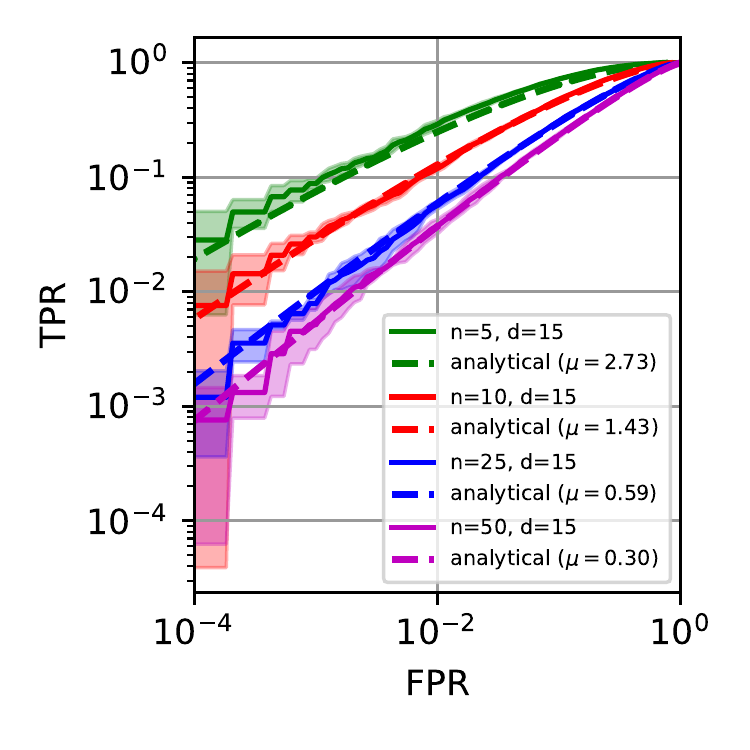}\capvsp
\caption{Changing $n$ (unif. dist.)}
\label{fig:util1aappf}
\end{subfigure}
\begin{subfigure}{0.24\textwidth}
\centering
\includegraphics[width=\textwidth]{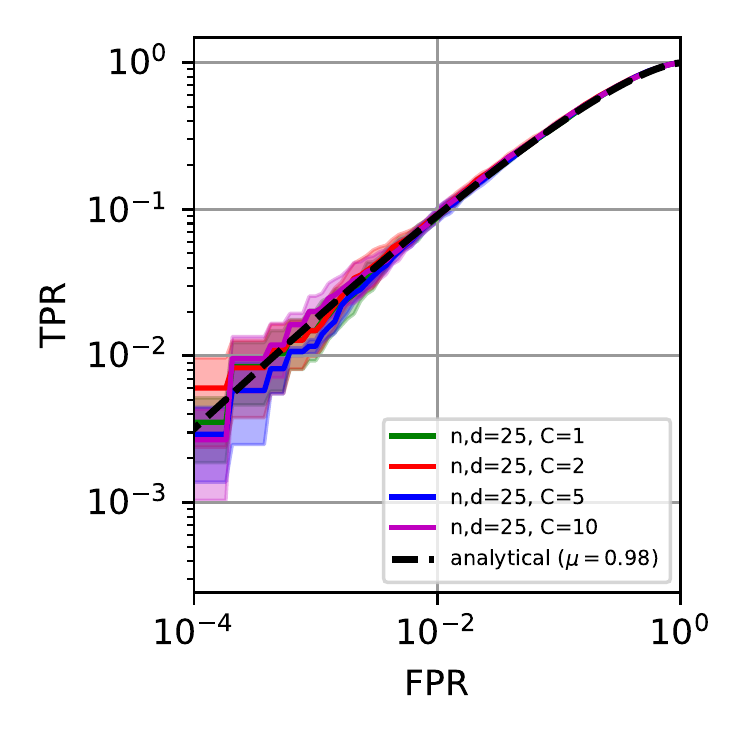}\capvsp
\caption{Changing $C$}
\label{fig:util2appg}
\end{subfigure}
\begin{subfigure}{0.24\textwidth}
\centering
\includegraphics[width=\textwidth]{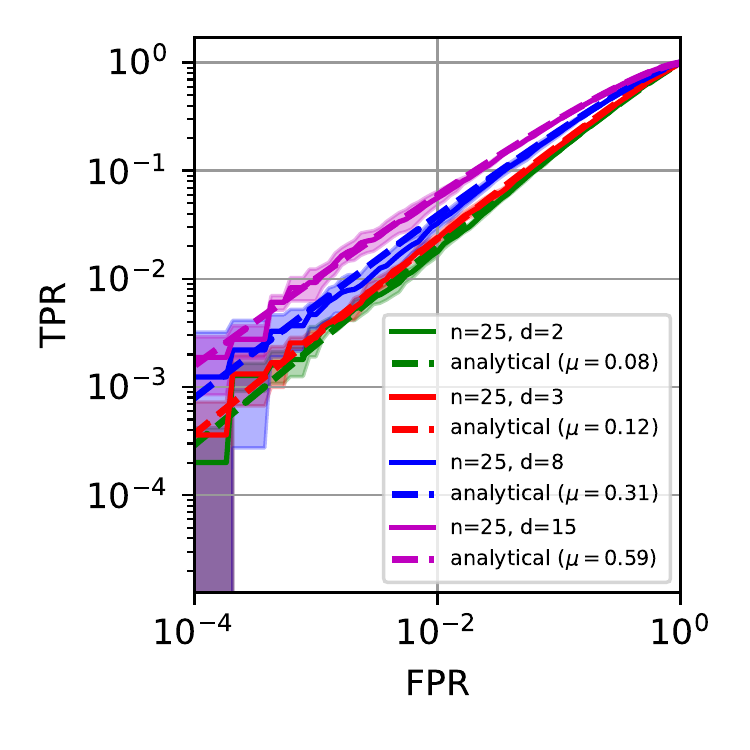}\capvsp
\caption{Changing $d$}
\label{fig:util3apph}
\end{subfigure}
\caption{\textbf{Ablation studies on approximation quality; Log-log plots in bottom row.}
\textbf{(a, e)}: Decreasing the batch size $n$ starts showing an effect when the batch size becomes as low as $n=5$, which is a batch size rarely used in practice. \textbf{(b, f)}: If we change the gradient distribution to be Uniform, there is no significant difference. This is expected as the CLT also holds for means of variables with bounded support. 
\textbf{(c, g)}: Changing the cropping threshold $C$ has no effect on the empirical predictions ($\mathbb{E}\left[\lVert\theta\rVert\right]=5$ in this example) \textbf{(d, h)}: Changing the gradient dimension $d$ only has a minor effect when $d=2$, which is a gradient dimension unlikely to be used in practice.
\label{fig:ablations}}
\end{figure}

\begin{figure}[htb]
\centering
\begin{subfigure}{0.32\textwidth}
\centering
\includegraphics[scale=0.60]{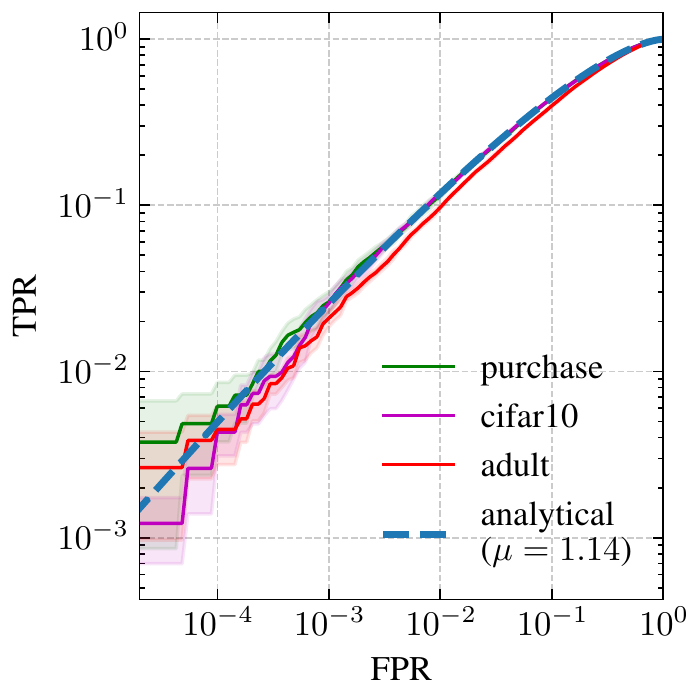}
\caption{Single step of simulated gradient distribution with known parameters.}
\label{fig:gradient1_loglog}
\end{subfigure}
\hfill
\begin{subfigure}{0.32\textwidth}
\centering
\includegraphics[scale=0.60]{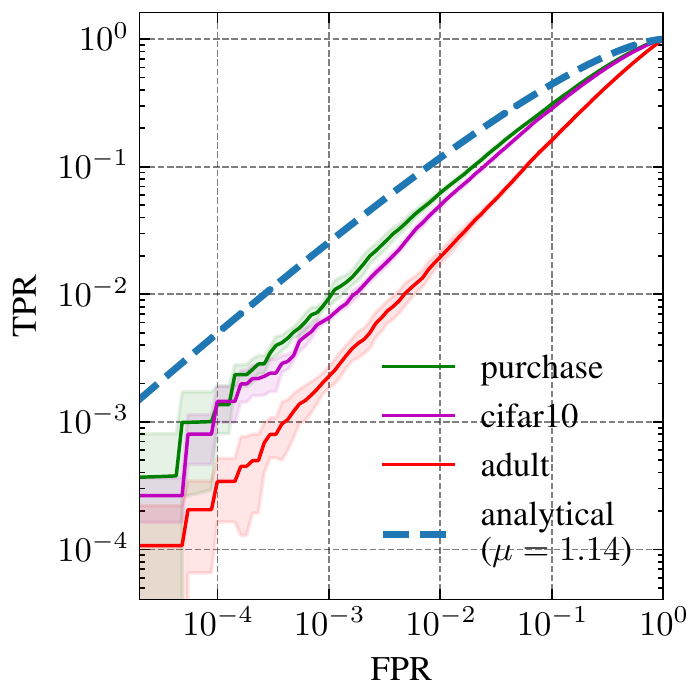}
\caption{Single step with real model gradients and estimated parameters.}
\label{fig:gradient2_loglog}
\end{subfigure}
\hfill
\begin{subfigure}{0.32\textwidth}
\centering
\includegraphics[scale=0.60]{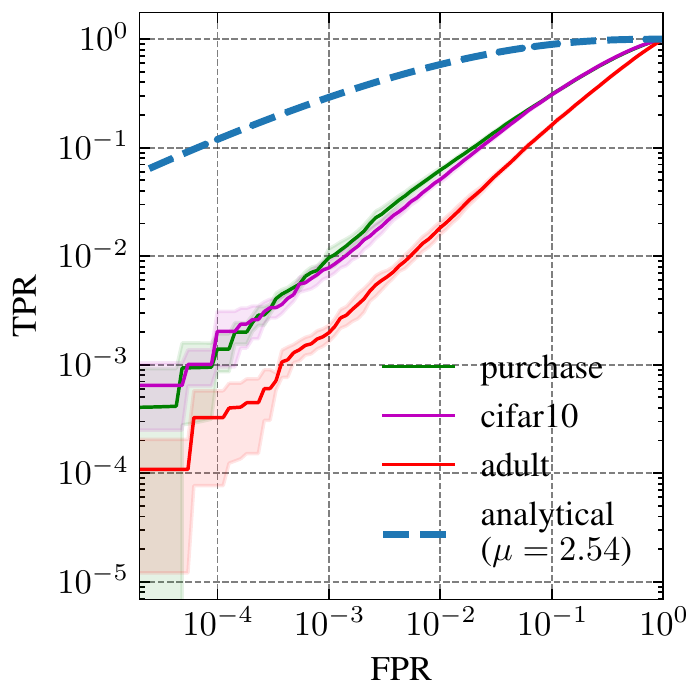}
\caption{As in (b), but now composition of 5 steps for real model gradients.}
\label{fig:gradient3_loglog}
\end{subfigure}
\caption{\textbf{Observed trade-off curves for the gradient attacks when $\tau^2=0$, Loglog-Scale}. We show trade-off curves when the gradient distribution is known (left) and when the gradients are obtained from a trained model that was finetuned on various data sets (center, right). The analytical solutions are computed with a value of $K=d$.}
\label{fig:verification_experimentlogscale}
\end{figure}

\subsection{Gradient Likelihood Ratio (GLiR) Attack}
\label{sec:app_gradientlrtattack}
We follow the common approach and trace the information flow from the data through the training process of stochastic gradient descent \citep{abadi2016deep, song2013stochastic}. 
We follow \citep{abadi2016deep, song2013stochastic} and make the standard assumption that only the mean over the individual gradients $\vm = \frac{1}{n}\sum_{i=1}^n \vtheta_i$, where $\vtheta_i \in \mathbb{R}^d$ is a sample gradient is used to update the model (or is published directly).
Consistent with the definition of the membership inference game, the attacker now tries to predict whether a specific gradient $\vtheta^\prime$ was part of the set $\left\{\vtheta_i\right\}_i$ that was used to compute the mean gradient $\vm$ or not.

An important requirement in the construction of the gradient likelihood ratio (GLiR) attack is the estimation of the true gradient mean $\vmu$ and the true inverse covariance matrix $\mSigma^{-1}$ since these quantities are essential parts of both the test statistic $S=\left(\vm-\vtheta'\right)^\top\mSigma^{-1}\left(\vm-\vtheta'\right)$ and the true gradient susceptibility term $K$ (see Proof of Theorem \ref{theorem:onestep_dp_sgd}). 
Here, we briefly summarize the attack algorithm (see \cref{alg:glir} for pseudo code):
\begin{enumerate}
\item The attacker uses their access to the data distribution, which is standard for membership inference attacks (see e.g., \citep{carlini2021membership,pawelczyk2022privacy}), to obtain estimates of $\mSigma$ and $\vmu$, which we refer to as $\hat{\mSigma}$ and $\hat{\vmu}$ where $\vmu$ and $\mSigma$ are the true means and covariances of the gradient distributions.
\item Given a gradient $\vtheta'$, the attacker uses $\hat{\mSigma}$ and $\hat{\vmu}$, and estimates $\hat{K}$.
\item Given a gradient $\vtheta'$, under the hypothesis that $\vtheta'$ is part of the test set, the attacker uses $\hat{K}$, $\hat{\mSigma}$ and $\hat{\vmu}$ to computes the quantiles of the non-central chi-squared distribution and compares them to the test statistic $S$, resulting in p-values. 
\item This procedure is repeated for several steps. The $p$-values can be aggregated through different means, where one strategy would be a simple multiplication of $p$-values as when assuming independence (corresponding to a sum of the $\log p$-values). However the threshold would have to be adjusted to compensate for multiple testing.
\item Finally, the attacker uses the p-values to determine whether a given gradient $\vtheta'$ was part of the training set or not. The full trade-off curve can then be obtained by varying the thresholds over the p-values.
\end{enumerate}

\subsection{Additional Membership Inference attacks}
\label{sec:app_additionalattacks}
We provide the same plots as in \Cref{fig:verification_experiment} using a loglog-scale in \Cref{fig:verification_experimentlogscale} to show that our bounds also hold for the low FPR regime.
Further, we can directly compare the analytical LRT gradient based attacks with the empirical loss based LRT attacks \citep{carlini2021membership}: we see that, at the false positive rate of $10^{-2}$, the loss-based attacks work substantially less reliably than our proposed analytical gradient based attacks (compare Figures \ref{fig:gradient3_loglog} and \ref{fig:loss_verification}).
The loss-based LRT attacks do not work at all when the model is trained for 5 steps only while our gradient attacks work up to 10 times more reliably.
In Figure \ref{fig:loss_utility}, we plot loss-based attacks on models that were trained for more steps than in Figure \ref{fig:loss_verification}.
In particular, these models were trained using the same number of epochs than the models from the utility experiment of Figure \ref{fig:utility}. 
Now, we see that the loss-based attacks slowly start working.

\begin{figure}[htb]
\centering
\begin{subfigure}{0.45\textwidth}
\centering
\includegraphics[scale=0.50]{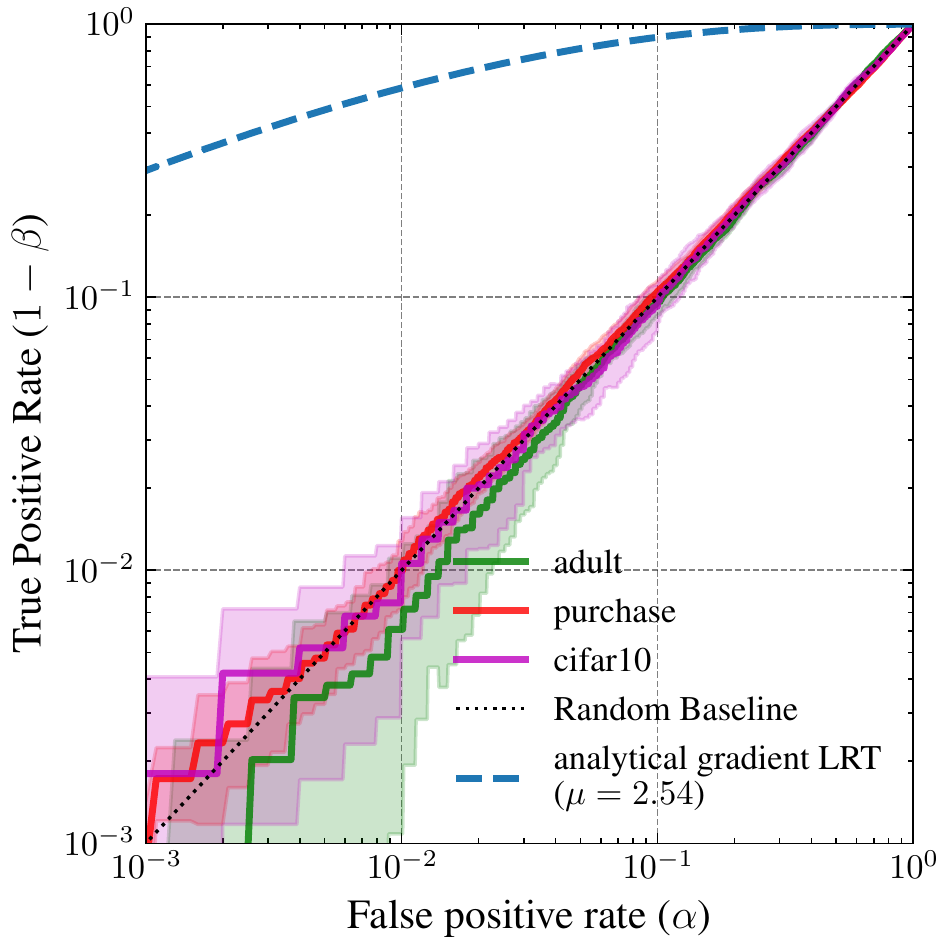}
\caption{Loss LRT attacks on the same models as in the verification experiment from Figures \ref{fig:verification_experiment} and \ref{fig:verification_experimentlogscale}.}
\label{fig:loss_verification}
\end{subfigure}
\hfill
\begin{subfigure}{0.45\textwidth}
\centering
\includegraphics[scale=0.50]{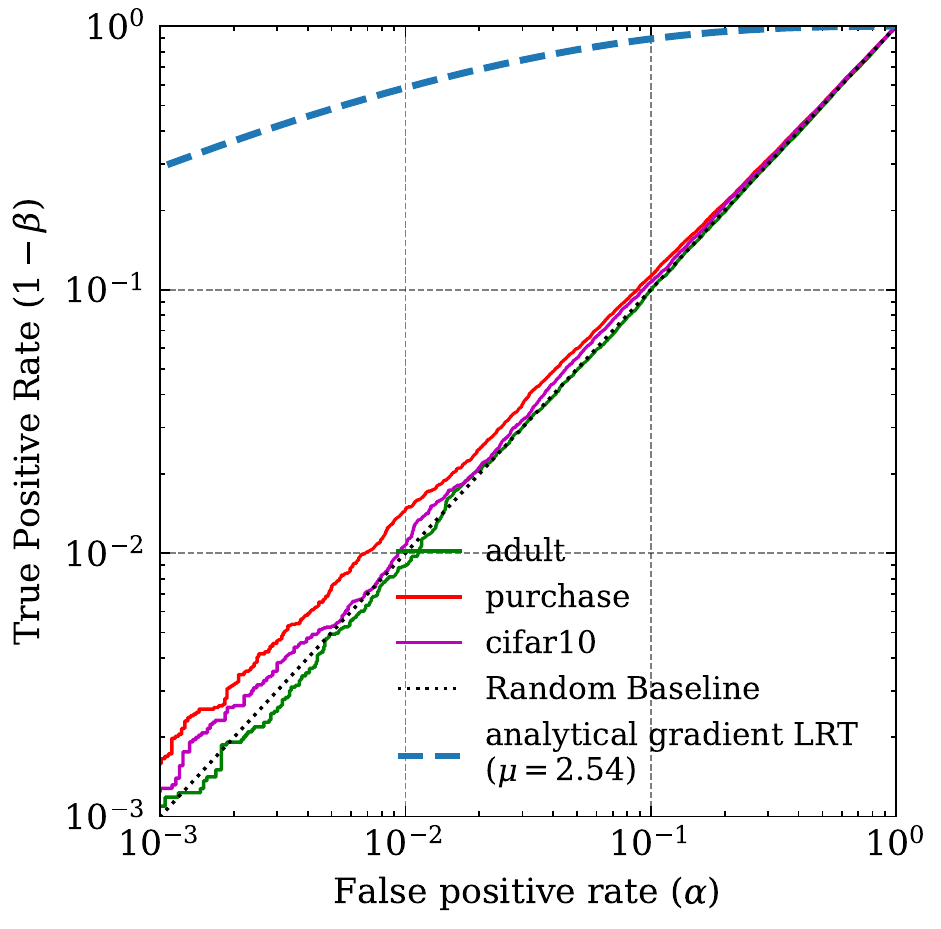}
\caption{Loss LRT attacks on the same models as in the utility experiment from Figure \ref{fig:utility}.}
\label{fig:loss_utility}
\end{subfigure}
\caption{\textbf{Observed trade-off curves for the empirical loss based LRT attacks by \citet{carlini2021membership} when $\tau^2=0$, Loglog-Scale}. We compare the analytical trade-off curves from the gradient attack to the trade-off curves obtained from the empirical loss-based LRT attacks.
The analytical solutions are computed as in Figure \ref{fig:verification_experimentlogscale}.
}
\label{fig:loss_attacks_experimentlogscale}
\end{figure}

\section{Proof of Theorem \ref{theorem:stochastic_composition} and Results for General Hypothesis Test Calculus \citep{dong2022gaussian}}
\subsection{Properties of the stochastic composition operator (Theorem \ref{theorem:stochastic_composition})}
\label{app_sec:stochasticcomposition}
\textbf{Theorem 4.1} (Stochastic composition of trade-off functions). 
\textit{The stochastic composition $\bigotimes_{\vx \sim \mathcal{D}} h(\vx)$ of trade-off functions $h(\vx)$ maintains the characteristics of a trade-off function, i.e., (1) it is convex, (2) non-increasing, (3) $\left(\bigotimes_{\vx \sim \mathcal{D}} h(\vx)\right)(r) \leq 1 - r$ for all $r \in [0,1]$, and (4) it is continuous at $r=0$.}
\begin{proof}
\textbf{(1) convexity.}
We start by proving convexity (1). Let $0 \leq a \leq b \leq 1$ and let $\lambda \in \left[0,1\right]$ and define $H(r) \coloneqq \left(\bigotimes_{\vx \sim \mathcal{D}} h(\vx) \right)(r)$ for brevity.

We have
\begin{align}
    H(a) = \min_{\bar{\alpha} \in \mathcal{E}(a, \mathcal{D})} \left\{ \beta_h(\bar{\alpha})\right\}
\end{align}
and we denote the test-specific FPR function (TS-FPR) that reaches this minimum by $\bar{\alpha}_a(\vx) \in  \mathcal{E}(a, \mathcal{D})$ such that
\begin{align}
H(a) = \beta_h(\bar{\alpha}_a) =  \mathbb{E}_{\vx \sim \mathcal{D}} \left[h(\vx)(\bar{\alpha}_a(\vx))\right].
\end{align}
We can do the same for $b$ and find a TS-FPR function $\bar{\alpha}_b(\vx) \in  \mathcal{E}(b, \mathcal{D})$ such that 
\begin{align}
H(b) = \beta_h(\bar{\alpha}_b) =  \mathbb{E}_{\vx \sim \mathcal{D}} \left[h(\vx)(\bar{\alpha}_b(\vx))\right].
\end{align}
For any $\lambda \in [0,1]$, we can define the convex combination of the TS-FPR functions $\bar{\alpha}_{\lambda, a, b}(\vx) = \lambda\bar{\alpha}_a(\vx)+(1-\lambda)\bar{\alpha}_b(\vx)$ and see that
\begin{align}
    \mathbb{E}_{\vx \sim \mathcal{D}}\left[\lambda \bar{\alpha}_a(\vx) + (1-\lambda) \bar{\alpha}_b(\vx)\right] &= \lambda \mathbb{E}_{\vx \sim \mathcal{D}}\left[\bar{\alpha}_a(\vx)\right] + (1-\lambda)\left[\bar{\alpha}_b(\vx)\right]\\
    &=\lambda a + (1-\lambda) b
\end{align}
which implies that
\begin{align}
    \bar{\alpha}_{\lambda, a, b} \in \mathcal{E}(\lambda a + (1-\lambda)b, \mathcal{D}),
\end{align}
i.e., $\bar{\alpha}_{\lambda, a, b}$ is a valid TS-FPR function for a global type 1 error of $\lambda a + (1-\lambda) b$.
We can now chose the function $\bar{\alpha}_{\lambda, a, b}$ to bound the minimum which allows to complete the proof
\begin{align}
H(\lambda a + (1-\lambda)b) &= \min_{\bar{\alpha} \in \mathcal{E}(\lambda a + (1-\lambda)b, \mathcal{D})} \left\{ \beta_h(\bar{\alpha})\right\} \\
&\leq \beta_h(\bar{\alpha}_{\lambda, a, b})\\
&= \mathbb{E}_{\vx \sim \mathcal{D}}\left[h(\vx)\left(\lambda \bar{\alpha}_a(\vx) + (1-\lambda) \bar{\alpha}_b(\vx)\right)\right]\\
&\leq \mathbb{E}_{\vx \sim \mathcal{D}}\left[\lambda h(\vx)\left( \bar{\alpha}_a(\vx)\right) + (1-\lambda)h(\vx)\left(\bar{\alpha}_b(\vx)\right)\right]\label{eqn:convexitytrade-off}\\
&= \lambda\mathbb{E}_{\vx \sim \mathcal{D}}\left[ h(\vx)\left( \bar{\alpha}_a(\vx)\right)\right] + (1-\lambda)\mathbb{E}_{\vx \sim \mathcal{D}}\left[h(\vx)\left(\bar{\alpha}_b(\vx)\right)\right]\\
&= \lambda H(a) + (1-\lambda) H(b).
\end{align}
In this derivation we use the convexity of the trade-off function $h(\vx)$ to arrive at \Cref{eqn:convexitytrade-off}.

\textbf{\textbf{(3) upper bounded by} $H(r) \leq 1-r$}. We prove property (3) next. For $r \in [0,1]$, we have 
\begin{align}
    H(r) = \min_{\bar{\alpha} \in \mathcal{E}(r, \mathcal{D})} \left\{ \beta_h(\bar{\alpha})\right\}
\end{align} where we can bound
\begin{align}
\beta_h(\bar{\alpha}) = \mathbb{E}_{\vx^\prime \sim \mathcal{D}} \left[h(\vx)(\bar{\alpha}(\vx))\right] \leq \mathbb{E}_{\vx^\prime \sim \mathcal{D}} \left[1-\bar{\alpha}(\vx)\right] = 1 -\mathbb{E}_{\vx \sim \mathcal{D}}\left[\bar{\alpha}(\vx)\right] = 1-r,
\end{align}
where we use the fact that $\mathbb{E}_{\vx \sim \mathcal{D}}\left[\bar{\alpha}(\vx)\right] = r$ for $\bar{\alpha} \in \mathcal{E}(r, \mathcal{D})$. Therefore,
\begin{align}
    H(r) = \min_{\bar{\alpha} \in \mathcal{E}(r, \mathcal{D})} \left\{ \beta_h(\bar{\alpha})\right\} \leq 1-r.
\end{align}

\textbf{(2) non-increasing}. Take any two points $0 \leq a < b \leq 1$. We can establish that $H(1) = 0$ by verifying that $H(r)\geq 0$ and the upper bound property (3). As $a < b \leq 1$, we can express 
\begin{align}
    b= \lambda a + (1-\lambda) \cdot 1,
\end{align}
for some $\lambda \in [0,1)$.
From the convexity property of $H$, we infer that
\begin{align}
H(b) = H(\lambda a + (1-\lambda) \cdot 1) \leq \lambda H(a) +(1-\lambda) H(1) = \lambda H(a) \leq H(a),
\end{align}
or $H(a) \geq H(b)$. We therefore conclude that $H$ is non-increasing.

\textbf{(4) Continuity at $r=0$.}
We will use the common $\epsilon{,}\delta$-criterion of continuity. Thus, we will have to show that for every $\epsilon > 0$, there exists a $\delta >0$ such that for all $|x-0| < \delta$ (as the support of $H(x)$ is [0,1], this means $x <\delta$), we have $\lvert H(\delta)-H(0)\rvert < \epsilon$.

We first denote the value of the composed trade-off at $0$ by $H_0 \coloneqq H(0)$ and $h_0(\vx)\coloneqq h(\vx)(0)$. Denote the TS-FPR function that takes the minimum again by $\bar{\alpha}$. Because $\mathbb{E}_{\vx\sim \mathcal{D}}\left[\bar{\alpha}\right] = 0$ together with $\bar{\alpha} \geq 0$ implies $\bar{\alpha} = 0$ almost everywhere, we can equivalently express $H_0=\mathbb{E}_{\vx\sim \mathcal{D}} \left[h_0(\vx)\right]$. %

Now let $1 \geq \epsilon > 0$ (for $\epsilon > 1$, we can choose $\delta = 1$) and let $\epsilon^\prime \coloneqq \epsilon/4$. We first note that the individual trade-off functions $h(\vx)$ are increasing and continuous at $r{=}0$ themselves, which means that for every $h(\vx), \epsilon_0$  there exists a $d(\vx, \epsilon_0)>0$ such that for all $y<d(\vx, \epsilon_0)$ we have $h(\vx)(0)-h(\vx)(y) = h_0(\vx)-h(\vx)(y) < \epsilon_0$. Absolute values are not required because $h(\vx)$ is monotonously decreasing. As $\vx$ is stochastic, these $d$'s also follow a certain distribution.

We now choose $\delta(\epsilon) = d_{\epsilon^\prime}$ such that 
\begin{align}
P_{\vx\sim\mathcal{D}}\left(h_0(\vx) - h(\vx)(d_{\epsilon^\prime}/ \epsilon^\prime) \geq \epsilon^\prime\right)<\epsilon^\prime.
\label{eqn:deltaepsiloncont}
\end{align}
This means that the probability that the change in the trade-off function when going from 0 to $d_{\epsilon^\prime}/ \epsilon^\prime$ will exceed $\epsilon^\prime$, is bounded by $\epsilon^\prime$.
We can find such a $d_{\epsilon^\prime} >0 $ for every $\epsilon^\prime >0$ by conducting the following steps: Finding a $d(\vx, \epsilon^\prime)$ for each $\vx$. This is possible due to the continuity of the individual trade-offs. However, the values of $d(\vx, \epsilon^\prime)$ can grow arbitrarily large or small for some $\vx$. Therefore, we select the $1-\epsilon^\prime$-quantile $q_{1-\epsilon^\prime}$ of the distribution of $d(\vx, \epsilon^\prime)$ for $\vx \sim \mathcal{D}$ for which we certainly have $0 < q_{1-\epsilon^\prime} < 1$. We then choose $d_{\epsilon^\prime}= q_{1-\epsilon^\prime}\epsilon^\prime > 0$. Having found such value $d_{\epsilon^\prime}$ with the characteristic in \Cref{eqn:deltaepsiloncont} implies
\begin{align}
\mathbb{E}_\vx\left[H_0 - h(\vx)(d_{\epsilon^\prime}/\epsilon^\prime)\right] = \mathbb{E}_\vx\left[h_0(\vx) - h(\vx)(d_{\epsilon^\prime}/\epsilon^\prime)\right] < 2\epsilon^\prime,
\end{align}
as with a probability smaller than $\epsilon'$, the change can only be bounded by 1, for the other values, it is bounded by $\epsilon^\prime$. We now bound $H_0- H(y)$ for $y < d_{\epsilon^\prime}$ to show that $H_0-H(y) < \epsilon$. First we note that to obtain a global true positive rate of $d_{\epsilon^\prime}$, $P(\bar{\alpha} > d_{\epsilon^\prime}/\epsilon^\prime) \leq \epsilon^\prime$:
\begin{align}
H_0 - H(y) \leq& H_0 -  H(d_{\epsilon^\prime})\\
\leq& P(\bar{\alpha} \leq d_{\epsilon^\prime}/\epsilon) \mathbb{E}_\vx[h_0(\vx)- h(\vx)(\bar{\alpha}(\vx))|\bar{\alpha} \leq d_{\epsilon^\prime}/\epsilon]+\\
&P(\bar{\alpha} > d_{\epsilon^\prime}/\epsilon) \mathbb{E}_\vx[h_0(\vx)- h(\vx)(\bar{\alpha}(\vx))|\bar{\alpha} > d_{\epsilon^\prime}/\epsilon]\\
\leq& 1 \cdot \mathbb{E}_\vx[H_0 - h(\vx)(\bar{\alpha}(\vx))|\bar{\alpha} \leq d_{\epsilon^\prime}/\epsilon] + \epsilon^\prime 
\end{align}
We also note that:
\begin{align}
&\mathbb{E}_\vx[h_0(\vx) - h(\vx)(d_{\epsilon^\prime}/\epsilon)|\bar{\alpha} \leq d_{\epsilon^\prime}/\epsilon]\\
&= \frac{\mathbb{E}_\vx[h_0(\vx) -h(\vx)(d_{\epsilon^\prime}/\epsilon)] - P(\bar{\alpha} > d_{\epsilon^\prime}/\epsilon)\mathbb{E}_\vx[h_0(\vx) -h(\vx)(d_{\epsilon^\prime}/\epsilon)|\bar{\alpha}\leq d_{\epsilon^\prime}/\epsilon]}{P(\bar{\alpha} \leq d_{\epsilon^\prime}/\epsilon)}\\
& < \frac{2\epsilon^\prime}{1-\epsilon^\prime} < 2\epsilon^\prime\frac{4}{3} < 3\epsilon^\prime.
\end{align}
In total we arrive at 
$H(0) - H(y) \leq H(0) - H(d_{\epsilon^\prime}) < 3\epsilon^\prime +\epsilon^\prime = 4\epsilon^\prime = \epsilon$.
\end{proof}


\subsection{A composition lemma for individual tests}
We repeat a lemma from Dong et al. \cite{dong2022gaussian}.
\begin{definition}
The tensor product of two trade-off functions $f = \text{Test}(P; Q)$ and $g = \text{Test}(P^\prime;Q^\prime)$ where $P, P^\prime, Q, Q^\prime$ are distributions is
defined as
\begin{align}
f \otimes g \coloneqq \text{Test}(P \times P^\prime;Q \times Q^\prime).
\end{align}
\end{definition}
Thus, the trade-off function $f$ of a test that is composed of independent dimension-wise tests with trade-off functions $f_1, \ldots, f_n$ can be written as $f= f_1 \otimes f_2 \otimes \ldots \otimes f_n$.
We reiterate the following result:
\begin{lemma} If there are two trade-off functions $f_1 \geq f_2$, i.e., the test $f_1$ is uniformly at least as hard as $f_2$, for any other trade-off function $g$:
\begin{align}
f_1 \otimes g \geq f_2  \otimes g.
\end{align}
\label{app_lem:dimensioncomposition}
\end{lemma}
Thus, by making an individual test harder, the functional composition (the tensor product, see below) will also be uniformly harder or maintain its hardness.
This lemma corresponds to Lemma C.2. of \citet[Appendix C]{dong2022gaussian} where the corresponding proof can be found.

\subsection{Composition results for $f$-MIP derived from results for $f$-DP}
\label{sec:app_compositionresults}
The result given in \Cref{lem:composition} follows from Theorem 11 and Lemma 3 in \citet{dong2022gaussian}, which provide functional composition results (we refer to the subsequent execution of two algorithms as \emph{functional} composition and the to the stochastic selection of a test as in \Cref{def:stochastic_composition} as \emph{stochastic} composition) for general hypotheses tests. In particular, the following result \cite[Theorem 4]{dong2022gaussian} can be restated:
\begin{theorem}
Let $A_i: D \times D_1 \times  \ldots \times D_{i-1}\rightarrow D_{i}$ be a series of $f_i$-DP algorithms (``mechansisms'' in \cite{dong2022gaussian})
for all inputs $x \in D, y_1 \in D_1, …, y_{i-1} \in  D_{i-1}$, for $i=1, \ldots,r$. Then the $r$-fold
composed mechanism  $M : D \rightarrow D_1 \times  \ldots \times D_r$, defined as $M=(A_1(x), A_2(x, A_1(x)), \ldots, A_r(x, ...))$ is $f_1 \otimes \ldots \otimes f_r$-DP.
\end{theorem}

\textbf{This and other composition results from Dong et al.\ also apply to the MIP bounds derived for our analysis of DP-SGD and for a stricter from of $f$-MIP, where additionally each $\text{Test}(A_0$, $A_1(\vx^\prime))$ $> f$ (in \Cref{def:f-mip}) is bounded by $f$ for each $\vx^\prime$. We can then replace $f$-DP with $f$-MIP in the results.} 

\textbf{Proof Scheme.} This can be seen as follows: Suppose we have several steps $i=1,\ldots,k$ each being naturally $f_i$-MIP as each of the tests in \Cref{def:f-mip} is bounded by $f_i$ for each $\vx^\prime$. We can apply the functional composition results from Dong et al. for each $\vx^\prime$ independently, as they hold for hypotheses tests in general. We therefore bound the functional composition for each individual $\vx^\prime$ through the composition result $\text{FuncComp}(f_1, \ldots, f_k)$. When we finally perform the stochastic composition to obtain $f$-MIP, we use the result stated below (\Cref{app_thm:worstcasestoasticcomp}) which tells us that if each individual test in the stochastic composition is bounded through some trade-off, this in an upper bound on the entire stochastic composition as well. In our case each individual test is bounded by $\text{FuncComp}(f_1, \ldots, f_k)$, which will thus bound the stochastic composition result and the level of $f$-MIP as well.

In our analysis of SGD, such a worst-case exists and boils down to the test for the highest value of $K$ considered (as this results in the highest $\mu$). Using these insights, the results by \citet{dong2022gaussian} can generally be transferred without further ramifications. Instead of using a value of $\frac{1}{\sigma}$ for the privacy level of each step, we plug in $\mu_{\text{step}}$ and arrive at \Cref{lem:composition}.

\textbf{Example: Composing $\mu$-GMIP algorithms.} As an alternative example for a composition results, we note that Dong et al. \cite[Corollary 2]{dong2022gaussian} provide a result which explicitly states that the $r$-fold composition of steps which are $\mu_i$-GDP each is
$\sqrt{\mu_1 + \ldots + \mu_r}$-GDP. This result also holds for $\mu$-GMIP as well, if the stochastic composition operator can again be bounded by the trade-off $g_{\mu_i}$ for every $\vx^\prime$, such as in the SGD case we study in this work. We use this result to arrive at the privacy level shown in  \Cref{fig:gradient3}, where we conduct $k$ steps of SGD, resulting in a combined privacy level of $\mu = \sqrt{k}\mu_{\text{step}}$.

We now consider the concluding remark. We first note that it is possible to make the functional composition harder by using smaller batch sizes than allowed in the theorem (see \citet[Theorem 9, Fact 1]{dong2022gaussian} to verify that the trade-off function will be at least as hard when the subsampling ratio is smaller). This is what we do for the steps up to $T$. For the subsequent steps $t>T$, we would be allowed so use the same mechanism with $\mu_{\text{step}}$-MIP but at a smaller batch size (this would usually require additional noise to be added). However, we do not perform these steps, thereby again making the functional composition of the tests harder.

\subsection{Worst-case bound for stochastic composition}
\begin{theorem}[Worst-case bounds for stochastic composition]
Let $h: \mathcal{X} \rightarrow \mathcal{F}$ denote a mapping from the input space to a set of trade-off functions $\mathcal{F}$. Suppose there is a trade-off function $f^{*}$ such that every other trade-off function $f \in \mathcal{F}$ is uniformly at least as hard as $f^{*}$,
\begin{align}
    f \geq f^{*}, \forall f\in \mathcal{F}.
\end{align}
Then, the stochastic composition of trade-off functions will also be uniformly at least as hard as $f^{*}$, i.e.
\begin{align}
\left(\bigotimes_{\vx \sim \mathcal{D}} h(\vx) \right) \geq f^{*}
\end{align}
regardless of the choice of $h$ or the distribution $\mathcal{D}$.
\label{app_thm:worstcasestoasticcomp}
\end{theorem}
\begin{proof}
Denote $H(r) \coloneqq \left(\bigotimes_{\vx \sim \mathcal{D}} h(\vx) \right)(r)$ again for brevity.
\begin{align}
    H(r) = \min_{\bar{\alpha} \in \mathcal{E}(r, \mathcal{D})} \left\{ \beta_h(\bar{\alpha})\right\}
\end{align} where for every $\bar{\alpha} \in \mathcal{E}(r, \mathcal{D})$ we can bound
\begin{align}
\beta_h(\bar{\alpha}) = \mathbb{E}_{\vx^\prime \sim \mathcal{D}} \left[h(\vx)(\bar{\alpha}(\vx))\right] \geq \mathbb{E}_{\vx^\prime \sim \mathcal{D}} \left[f^{*}\left(\bar{\alpha}(\vx)\right)\right] \geq f^{*}\left(\mathbb{E}_{\vx^\prime \sim \mathcal{D}}\left[\bar{\alpha}(\vx)\right]\right) = f^{*}(r).
\end{align}
We use the Jensens inequality to derive that $\mathbb{E}_{\vx^\prime \sim \mathcal{D}} \left[f^{*}\left(\bar{\alpha}(\vx)\right)\right] \geq f^{*}\left(\mathbb{E}_{\vx^\prime \sim \mathcal{D}}\left[\bar{\alpha}(\vx)\right]\right)$ because $f^{*}$ is convex. Therefore, we conclude that for every $r \in [0, 1]$
\begin{align}
    H(r) = \min_{\bar{\alpha} \in \mathcal{E}(r, \mathcal{D})} \left\{ \beta_h(\bar{\alpha})\right\} \geq f^{*}(r).
\end{align}
\end{proof}

\section{Proof of Theorem \ref{theorem:onestep_dp_sgd} and Corollary \ref{corollary:onestep_dp_sgd}}
\label{sec:app_onestep_dp_sgd}
For the sake of better readability, we summarize our setting before we proceed with the formal proof:
\begin{itemize}
\item $\Theta = [\vtheta_1, \ldots, \vtheta_n], \vtheta_i \sim P$, where $\vtheta_i \in \mathbb{R}^d$. $P$ can be any distribution with finite mean $\vmu \in \mathbb{R}^d$ and covariance $\mSigma \in \mathbb{R}^{d\times d}$ (in our application, $\theta_i$  are gradients of the samples)
\item The sample mean $\vm = \frac{1}{n}\sum_{i=1}^n \vtheta_i$ is published
\item $b \in \{0,1\}$ is drawn uniformly at random. If $b = 0, \vx' \sim \Theta$, if $b=1$,  $\vx' \sim P$. $\vx'$ is published.
\item The attacker $\mathcal{A}(\vm, \vx')=b'$ attempts to predict the value of $b$, i.e., whether $\vx'$ was in the training set or not.
\end{itemize}
The following hypothesis test succinctly summarizes the attacker's problem in the this setting\footnote{Note that the hypotheses are interchanged with respect to the main paper here. Following the remarks after Corollary 2 of Dong et al., the trade-off function is inverted by when interchanging $H_0$, $H_1$. To arrive at the trade-off in the main paper, we will later invert the trade-off function derived here.}:
\begin{align}
&H_0: \vx' \text{ was drawn from } \mTheta &
&H_1: \vx' \text{ was drawn from } P.
\label{eq:hypothesis_test_exact_normal}
\end{align}
Based on this testing setup, the attacker constructs an attack based on a likelihood ratio test.
Next, we summarize the individual proof steps before we give the formal proof:
\begin{enumerate}
\item Derive the distributions of $\vm$ under the null and the alternative hypothesis for a given $\vx^\prime$.
\item Given these two distributions, we can setup the likelihood ratio, which will yield the test statistic $S$. This statistic will be optimal due to the Neyman-Pearson fundamental testing lemma;
\item Given the test statistic $S$ and the distribution under the null hypothesis, we derive the rejection region of the likelihood ratio test for a given level $\alpha$;
\item Finally, we derive the false negative rate $\beta(\alpha)$ of the likelihood ratio test. This trade-off function depends on the CDF and inverse CDF of non-central $\chi^2$ distribution;
\item In a final step, we provide approximations to the trade-off function for large $d$.
\item Steps 1-5 initially prove f-membership inference privacy of a single SGD step when we add Gaussian noise with covariance $\hat{\tau}^2\mSigma$, where $\mSigma$ is the covariance of the gradients that we would like to privatize (see \Cref{sec:app_onestep_dp_sgddatanoise}).
We then use this result to bound the privacy of SGD with unit noise in \Cref{sec:app_onestep_dp_sgdunitnoise}. 
There we show that the result from \Cref{sec:app_onestep_dp_sgddatanoise} can be used in combination with a scaled noise level to guarantee membership inference privacy when we add Gaussian noise with covariance $\tau^2\mI$.
\end{enumerate}

\subsection{Proof of Theorem \ref{theorem:onestep_dp_sgd} and Corollary \ref{corollary:onestep_dp_sgd} with data dependent noise}
\label{sec:app_onestep_dp_sgddatanoise}
In this section, we consider the effect that averaging with Gaussian noise has on membership inference privacy. We first add Gaussian noise with covariance $\hat{\tau}^2\mSigma$, where $\mSigma$ is the covariance of the gradients that we would like to privatize. In the next section, we will consider the case of independent unit noise $\tau^2\mI$, which can be derived from the result presented here.
The proof in this subsection follows the steps outlined in the previous section.
\begin{proof}
\textbf{Step 1: Deriving the distributions of $\vm$ under $H_0$ and $H_1$.}
First, we derive the distributions of $\vm$ for both cases of interest. We suppose that the number of averaged samples is sufficiently large such that we can apply the Central Limit Theorem. Note that this does not restrict the form of the distribution $P$, besides having finite variance.
Below, we start with the distribution of $\vm$ under $H_0$ (with no additional noise yet):
\begin{align}
\vm \sim \mathcal{N}\left(\frac{1}{n}\vx' +\frac{n-1}{n}\vmu, \frac{(n-1)}{n^2}\mSigma\right) = \mathcal{N}\left( \vmu + \frac{1}{n}(\vx{'}-\vmu), \frac{(n-1)}{n^2}\mSigma\right).
\end{align}
Moreover, under the alternative hypothesis $H_1$, we have:
\begin{align}
\vm \sim \mathcal{N}\left(\vmu, \frac{1}{n}\mSigma\right).
\end{align}
Instead of testing the distributions of $\vm$ we directly, we can equivalently test $\vm-\vx^\prime$ by subtracting $\vx^\prime$ from both means. Adding Gaussian noise $Y \sim \mathcal{N}(\mathbf{0}, \hat{\tau}^2 \mSigma)$ under both hypotheses results in the test that we provide below:
\begin{align}
& \bigotimes_{\vx^\prime \sim \mathcal{D}}\text{Test}\left[\mathcal{N}\left(\frac{n-1}{n}\vmu, \frac{(n-1)}{n^2}\mSigma\right)-\frac{n-1}{n}\vx^\prime + Y, \mathcal{N}\left(\vmu, \frac{1}{n}\mSigma\right)-\vx^\prime + Y\right].
\label{eqn:addsigmanoise}
\end{align}
Again, we now consider the test for a fixed $\vx^\prime$. If we can show that there is one $\vx^\prime$ that makes the test harder than any other $\vx^{\prime} \in \mathcal{X}$, we can apply \Cref{app_thm:worstcasestoasticcomp} and show that the composed test is uniformly at least as hard as for $\vx^\prime$. 

We conduct the following reformulations (note that the hardness of a test remains unaffected by by invertible transforms, e.g., linear transforms):
\begin{align}
& \text{Test}\left[\mathcal{N}\left(\frac{n-1}{n}\vmu, \frac{(n-1)}{n^2}\mSigma\right)-\frac{n-1}{n}\vx^\prime+ Y, \mathcal{N}\left(\vmu, \frac{1}{n}\mSigma\right)-\vx^\prime +Y\right]\label{eqn:worstcasetestpnoise}\\
& \Longleftrightarrow \text{Test}\left[\mathcal{N}\left(\frac{n-1}{n}\left(\vmu-\vx^\prime\right), \frac{(n-1)}{n^2}\mSigma + \hat{\tau}^2 \mSigma \right), \mathcal{N}\left(\vmu -\vx^\prime, \frac{1}{n}\mSigma + \hat{\tau}^2 \mSigma \right)\right] \\
& \Longleftrightarrow \text{Test}\left[\mathcal{N}\left(-\frac{1}{n}\left(\vmu-\vx^\prime\right), \bigg(\frac{(n-1)}{n^2} + \hat{\tau}^2 
\bigg) \mSigma\right), \mathcal{N}\left(\mathbf{0}, \bigg(\frac{1}{n} + \hat{\tau}^2 \bigg) \mSigma\right)\right] \label{eqn:subtractmu} \\
& \Longleftrightarrow \text{Test}\left[\mathcal{N}\left(-\frac{1}{n\sqrt{n^{-1} + \hat{\tau}^2}}\mSigma^{-\frac{1}{2}}\left(\vmu-\vx^\prime\right), \left(\frac{n-1}{n^2} + \hat{\tau}^2\right) \left(\frac{1}{n} + \hat{\tau}^2\right)^{-1} \mI\right), \mathcal{N}\left(\mathbf{0}, \mI\right)\right] \label{eqn:inversesigma} \\
& \Longleftrightarrow \text{Test}\left[\mathcal{N}\left(-\frac{1}{n \sqrt{n^{-1} + \hat{\tau}^2}} \tilde{\vdelta} , \frac{n(1+\hat{\tau}^2 n) - 1}{n(1+\hat{\tau}^2 n)} \mI\right), \mathcal{N}\left(\mathbf{0}, \mI\right)\right],
\end{align}
where $\tilde{\vdelta}=\mSigma^{-\frac{1}{2}}\left(\vmu-\vx^\prime\right) = \tilde{\vmu} - \tilde{\vx}^\prime$ and the tilde indicates the corresponding quantities transformed by $\Sigma^{-\frac{1}{2}}$.
For instance, to arrive at \Cref{eqn:subtractmu} and \Cref{eqn:inversesigma} the random variable $\vm - \vx^\prime$ is transformed by subtracting $(\vmu - \vx^\prime)$ and multiplied by $\frac{1}{\sqrt{n^{-1}+ \hat{\tau}^2}}\mSigma^{-\frac{1}{2}}=\sqrt{\frac{n^2}{n+ n^2\hat{\tau}^2}}\mSigma^{-\frac{1}{2}}$, respectively. This yields the following transformed likelihood ratio test for the transformed random variable
\begin{align}
Q &\coloneqq\sqrt{\frac{n^2}{n+n^2 \hat{\tau}^2}}\mSigma^{-\frac{1}{2}} \big((\vm - \vx^\prime) - (\vmu - \vx^\prime)\big)=\sqrt{\frac{n^2}{n+n^2 \hat{\tau}^2}}\mSigma^{-\frac{1}{2}}\big(\vm-\vmu)\\ &= \sqrt{\frac{n^2}{n+n^2 \hat{\tau}^2}}(\tilde{\vm} - \tilde{\vmu}).
\end{align}

\textbf{Step 2: Deriving the test statistic.}
By the Neyman-Pearson Lemma \cite[Theorem 3.2.1.]{lehmann2005testing}, conducting the likelihood ratio test will be most powerful test at a given false positive rate. 
The corresponding likelihood ratio is given as follows:
\begin{align}
\text{LR} &= \frac{p_0(Q)}{p_1(Q)}  = \frac{\mathcal{N}\left(Q;\vmu_1, \sigma^2_1 \mI\right)}{\mathcal{N}\left(Q; \vmu_2, \sigma^2_2 \mI\right)} \\
& = \frac{\mathcal{N}\left(Q;-\frac{1}{\sqrt{n+n^2\hat{\tau}^2}}\mSigma^{-\frac{1}{2}}\left(\vmu-\vx^\prime\right), \frac{n+n^2\hat{\tau}^2-1}{n+n^2\hat{\tau}^2} \mI\right)}{\mathcal{N}\left(Q;\mathbf{0}, \mI\right)} \\
& = c_2 \mathcal{N}\left(Q; \vd, \mD\right).
\label{eqn:gaussianformlrt}
\end{align}
We can use identities for the ratio of two normal distributions (with $\vmu_1=-\frac{1}{\sqrt{n(1+n\hat{\tau}^2)}} \tilde{\vdelta}, \vmu_2 = \mathbf{0}, \mSigma_1 = \frac{n(1+\hat{\tau}^2 n) - 1}{n(1+\hat{\tau}^2 n)}\mI, \mSigma_2 = \mI$) and obtain
\begin{align}
\mD= (\mSigma_1^{-1}-\mSigma_2^{-1})^{-1}=\left(\frac{n(1+\hat{\tau}^2 n) - 1}{n(1+\hat{\tau}^2 n)} -\frac{n(1+\hat{\tau}^2 n)}{n(1+\hat{\tau}^2 n)}\right)^{-1} \mI = (n + \hat{\tau}^2 n^2 - 1)\mI
\end{align}
and 
\begin{align}
\vd = \mD\left(\mSigma_1^{-1}\mu_1 - \mSigma_2^{-1}\mu_2\right) & = (n+\hat{\tau}^2 n^2-1)\frac{n(1+\hat{\tau}^2 n) }{n(1+\hat{\tau}^2 n)-1} \left(-\frac{1}{\sqrt{n+\hat{\tau}^2 n^2}}\tilde{\vdelta}\right) \\
& = -\sqrt{n + \hat{\tau}^2 n^2 }\tilde{\vdelta} = -\sqrt{n + \hat{\tau}^2 n^2 }\left(\tilde{\vmu} - \tilde{\vx}^\prime\right).
\end{align}
For a Gaussian likelihood ratio of the form in \Cref{eqn:gaussianformlrt}, i.e., $S = c_2\exp\left(-\frac{1}{2}(\vs^\prime-\vd)^\top\mD^{-1}(\vs^\prime-\vd)\right)$, where $\vs^\prime=\sqrt{\frac{n^2}{n+n^2 \hat{\tau}^2}}(\tilde{\vm} - \tilde{\vmu})$ it suffices to use the inner argument as a test statistic, as exp is an invertible transform. Therefore, we can use the following as a test statistic:
\begin{align}
S &= \sqrt{\frac{n^2}{n+n^2 \hat{\tau}^2}} \left( (\tilde{\vm} - \tilde{\vmu}) + \frac{n +\hat{\tau}^2n^2}{n}\tilde{\vdelta}\right)^\top \frac{1}{(n + n^2\hat{\tau}^2)-1} \mI \sqrt{\frac{n^2}{n+n^2 \hat{\tau}^2}} \left( (\tilde{\vm} - \tilde{\vmu}) + \frac{n +\hat{\tau}^2n^2}{n}\tilde{\vdelta}\right) \\
&=\frac{n^2}{((n + n^2\hat{\tau}^2)-1)(n + n^2\hat{\tau}^2)} \left( (\tilde{\vm} - \tilde{\vmu}) + \frac{n +\hat{\tau}^2n^2}{n}\tilde{\vdelta}\right)^\top\left( (\tilde{\vm} - \tilde{\vmu}) + \frac{n +\hat{\tau}^2n^2}{n}\tilde{\vdelta}\right), 
\end{align}
for which we can derive the closed-form distributions under both the null and alternative hypotheses.

\textbf{Step 3: Deriving the distributions of $S$ under $H_0$ and $H_1$.}
From above, we know that, under the respective hypotheses we have:
\begin{align}
H_0 &: \sqrt{\frac{n^2}{n+n^2 \hat{\tau}^2}} (\tilde{\vm} - \tilde{\vmu}) \sim \mathcal{N}\bigg(-\frac{1}{\sqrt{n+n^2\hat{\tau}^2}}\tilde{\vdelta}, \frac{n+n^2\hat{\tau}^2-1}{n+n^2\hat{\tau}^2} \mathbf{I}\bigg) \\
H_1 &: \sqrt{\frac{n^2}{n+n^2 \hat{\tau}^2}} (\tilde{\vm} - \tilde{\vmu}) \sim \mathcal{N}\big(\mathbf{0}, \mathbf{I}\big)
\end{align}
and
\begin{align}
H_0 &: (\tilde{\vm} - \tilde{\vmu}) \sim \mathcal{N}\bigg(-\frac{1}{n}\tilde{\vdelta}, \frac{n+n^2\hat{\tau}^2-1}{n^2} \mathbf{I}\bigg) \\
H_1 &:  (\tilde{\vm} - \tilde{\vmu}) \sim \mathcal{N}\big(\mathbf{0}, \frac{n+n^2\hat{\tau}^2}{n^2}\mathbf{I}\big).
\end{align}
and
hence, for $\tilde{\vl} = (\tilde{\vm} - \tilde{\vmu}) + \frac{n+n^2\hat{\tau}^2}{n}\tilde{\vdelta}=(\tilde{\vm} - \tilde{\vmu}) + \frac{n+n^2\hat{\tau}^2}{n}(\tilde{\vmu}-\tilde{\vx}^\prime)$ the distributions are given by:
\begin{align}
H_0 &: \tilde{\vl} \sim \mathcal{N}\bigg(\frac{n-1}{n}(\tilde{\vmu} - \tilde{\vx}^\prime) + n\hat{\tau}^2 (\tilde{\vmu}-\tilde{\vx}^\prime), \frac{n+n^2\hat{\tau}^2-1}{n^2}\mathbf{I}\bigg) \\
H_1 &:\tilde{\vl} \sim \mathcal{N}\bigg((\tilde{\vmu} - \tilde{\vx}^\prime) + n\hat{\tau}^2 (\tilde{\vmu}-\tilde{\vx}^\prime), \frac{n+n^2\hat{\tau}^2}{n^2} \mathbf{I}\bigg).
\end{align}
Next, note that if $\tilde{\vl} ~ \sim \mathcal{N}(\vl, \kappa \mI)$, and $\frac{\tilde{\vl}}{\sqrt{\kappa}} \sim \mathcal{N}(\frac{\vl}{\sqrt{\kappa}},\mI)$ then we have that 
\begin{align}
 \lVert \tilde{\vl} \rVert_2^2 =  \kappa \sum_{j=1}^d (U_j + b_j)^2,
\end{align}
where $U_j$ are standard normal variables.
Hence, $\lVert \tilde{\vl} \rVert_2^2$ follows the law of a (scaled) non-central chi-squared distribution with $d$ degrees of freedom and $\vb = \frac{1}{\sqrt{\kappa}} \vl$, i.e.,
\begin{align}
 \frac{\lVert \tilde{\vl} \rVert_2^2 }{\kappa}\sim \chi_d'^{2}(\gamma), 
\end{align}
where $\gamma = \sum_{j=1}^d b_j^2 = \sum_{j=1}^d \big( \frac{1}{\sqrt{\kappa}} l_j \big)^2 = \frac{1}{\kappa} \sum_{j=1}^d l_j^2$.

Under the null hypothesis, we obtain $\kappa_0 = \frac{n+n^2\hat{\tau}^2-1}{n^2}$ and $\vl_0 = \frac{n-1+n^2\hat{\tau}^2}{n}(\tilde{\vmu}-\tilde{\vx}^\prime)$. Therefore, the distribution of the test statistic under the null hypothesis is given by the scaled, non-central chi-squared random variable:
\begin{align}
H_0: S &= \frac{n^2\kappa_0}{(n+n^2\hat{\tau}^2-1)(n+n^2\hat{\tau}^2)} \frac{\lVert \tilde{\vl} \rVert_2^2}{\kappa_0} \sim \frac{n^2\kappa_0}{(n+n^2\hat{\tau}^2-1)(n+n^2\hat{\tau}^2)} 
\chi_{d}'^2(\gamma_0)
\\
&\sim \frac{1}{n+n^2\hat{\tau}^2} 
\chi_{d}'^2(\gamma_0).
\end{align}
In summary, 
\begin{align}
(n+n^2\hat{\tau}^2) S  \sim  
\chi_{d}'^2(\gamma_0),
\end{align}
where $\gamma_0 = \lVert \vl_0 \rVert_2^2/\kappa_0$.

\textbf{Step 4: Deriving the rejection region for any $\alpha$.}
Therefore, the rejection region of the null hypothesis at a significance level of $\alpha$ can be formulated as:
\begin{align}
\left\{ (n + \hat{\tau}^2n^2) S \geq \text{CDF}_{\mathcal{X}'^2_d(\gamma_0)}^{-1}(1-\alpha)\right\} =
 \left\{S \geq  \frac{1}{n + \hat{\tau}^2n^2}\text{CDF}_{\mathcal{X}'^2_d(\gamma_0)}^{-1}(1-\alpha)\right\}.
\end{align}


\textbf{Step 5: Deriving the false negative rate $\beta(\alpha)$ at any $\alpha$.}
To compute the type two error rate $\beta$ (the null hypothesis is accepted, but the alternative $H_1$ is true), we compute the probability of mistakenly accepting it. 
To this end, note that $\kappa_1 = \frac{n+n^2\hat{\tau}^2}{n^2}$ and that $\vl_1=(1++n\hat{\tau}^2)(\tilde{\vmu} - \tilde{\vx}^\prime)$.
Now, we can derive the distribution of the test statistic under the alternative hypothesis:
\begin{align}
H_1: S = \frac{n^2\kappa_1}{(n+n^2\hat{\tau}^2-1)(n+n^2\hat{\tau}^2)} \frac{\lVert \tilde{\vl} \rVert_2^2}{\kappa_1} \sim \frac{1}{n+n^2\hat{\tau}^2-1} \chi_{d}'^2(\gamma_1),
\end{align}
where $\gamma_1 = \lVert \vl_1 \rVert_2^2/\kappa_1$.
Note that the form of this distribution, 
\begin{align}
(n+n^2\hat{\tau}^2-1 ) S  \sim \chi_{d}'^2(\gamma_1),
\end{align}

 where $\gamma_1=n\lVert\tilde{\vmu} - \tilde{\vx}^\prime\rVert_2^2\coloneqq n K$ when no noise is added ($\hat{\tau}=0$) is the statistic used in the GLiR attack (\Cref{alg:glir}).
The type two error is given by:
\begin{align}
\beta = & P_1\left(S \leq \frac{1}{n + n^2\hat{\tau}^2}\text{CDF}_{\mathcal{X}'^2_d(\gamma_0)}^{-1}(1-\alpha)\right) \\
=& P_1\left((n + n^2\hat{\tau}^2-1) S \leq \frac{n + n^2\hat{\tau}^2-1}{n + n^2\hat{\tau}^2}\text{CDF}_{\mathcal{X}'^2_d(\gamma_0)}^{-1}(1-\alpha)) \right) \notag \\
=& P\left(X \leq \frac{n + n^2\hat{\tau}^2-1}{n + n^2\hat{\tau}^2} \text{CDF}_{\mathcal{X}'^2_d(\gamma_0)}^{-1}(1-\alpha)) \right) \notag \\
=& \text{CDF}_{\mathcal{X}'^2_d(\gamma_1)}\left(\frac{n + n^2\hat{\tau}^2-1}{n + n^2\hat{\tau}^2}\text{CDF}_{\mathcal{X}'^2_d(\gamma_0)}^{-1}(1-\alpha)) \right) \notag \\
=& \text{CDF}_{\mathcal{X}'^2_d(\gamma_1)}\left(\frac{n + n^2\hat{\tau}^2-1}{n + n^2\hat{\tau}^2} \text{CDF}_{\mathcal{X}'^2_d(\gamma_0)}^{-1}(1-\alpha)) \right),
\end{align}
where $X\sim \chi_{d}'^2(\gamma_1)$.
\begin{remark}
Computing the inverse trade-off, i.e., solving for $\alpha$ would result in 
\begin{align}
\alpha = 1-\text{CDF}_{\mathcal{X}'^2_d(\gamma_0)}\left(\frac{n + n^2\hat{\tau}^2}{n + n^2\hat{\tau}^2-1}\text{CDF}^{-1}_{\mathcal{X}'^2_d(\gamma_1)}(\beta)\right),
\end{align}
giving the trade-off curve for the hypotheses as stated in main paper and resulting in \Cref{theorem:onestep_dp_sgd}.
For the subsequent analysis this interchange of $H_0$, $H_1$ does not play a role as the trade-off function is symmetric for large $d, n$ (it is its own inverse).
\end{remark} 
    

\textbf{Step 6: Large $d,n$ approximations.}
The following fact will be useful. Let $V \sim \chi_d'^{2}(\gamma)$, then
$\frac{V-(d + \gamma)}{\sqrt{2(d + 2\gamma)}} \to \mathcal{N}(0,1)$ when $d \to \infty$. 
The trade-off function for our hypothesis test can thus be expressed through the normal CDF $\Phi$ as: 
\begin{align}
\beta \approx& \Phi\left(\frac{\left(\frac{n+ \hat{\tau}^2n^2-1}{n+ \hat{\tau}^2n^2}\text{CDF}_{\mathcal{X}'^2_d(\gamma_0)}^{-1}(1-\alpha)\right)-(d + \gamma_1)}{\sqrt{2(d + 2 \gamma_1)}}\right) &\\
\approx& \Phi\left(\frac{\left(\frac{n+ \hat{\tau}^2n^2-1}{n+ \hat{\tau}^2n^2}\left(\sqrt{2(d + 2 \gamma_0)}\Phi^{-1}(1-\alpha) +(d + \gamma_0)\right)\right)-(d+\gamma_1)}{\sqrt{2(d + 2 \gamma_1)}}\right) &\\
=&\Phi \Bigg(\Phi^{-1}(1-\alpha) - \Phi^{-1}(1-\alpha) \frac{(n + n^2\hat{\tau}^2) s_1 - (n + n^2\hat{\tau}^2-1) s_0}{(n + n^2\hat{\tau}^2) s_1}&\nonumber\\
&- \frac{(n + n^2\hat{\tau}^2)m_1-(n + n^2\hat{\tau}^2-1)m_0}{(n + n^2\hat{\tau}^2)s_1} \Bigg),&
\end{align}
where $s_1 =\sqrt{2(d+2\gamma_1)}$, $s_0=\sqrt{2(d+2\gamma_0)}$, $m_0=d+\gamma_0$ and $m_1=d+\gamma_1$.
Thus we have:
\begin{align}
\beta &\approx& \Phi \Bigg(\Phi^{-1}(1-\alpha) - \Phi^{-1}(1-\alpha) \Bigg(1  - \frac{n+ \hat{\tau}^2n^2-1}{n+ \hat{\tau}^2n^2} \sqrt{\frac{d+2\gamma_0}{d+2\gamma_1}} \Bigg)\nonumber\\
&&- \frac{(n+ \hat{\tau}^2n^2)(d+\gamma_1) - (n+ \hat{\tau}^2n^2-1) (d+\gamma_0)}{(n+ \hat{\tau}^2n^2) \sqrt{2d + 4 \gamma_1)}} \Bigg),
\end{align}
For large $n$, $\frac{n+ \hat{\tau}^2n^2-1}{n+ \hat{\tau}^2n^2} \sqrt{\frac{d+2\gamma_0}{d+2\gamma_1}}\approx1$ and the second term can be dropped in the approximation.
Next, recall that $\gamma_0 = \frac{n^2 \lVert\vl_0\rVert^2}{n+n^2\hat{\tau}^2-1}$ and $\gamma_1 = \frac{n^2 \lVert\vl_1\rVert^2}{n+n^2\hat{\tau}^2}$:
\begin{align}
\beta &\approx \Phi \Bigg(\Phi^{-1}(1-\alpha) - \frac{(n+ \hat{\tau}^2n^2)(d+\frac{n^2 \lVert\vl_1\rVert^2}{n+n^2\hat{\tau}^2}) - (n+ \hat{\tau}^2n^2-1) (d+\frac{n^2 \lVert\vl_0\rVert^2}{n+n^2\hat{\tau}^2-1})}{(n+ \hat{\tau}^2n^2) \sqrt{2d + 4  \frac{n^2 \lVert\vl_1\rVert^2}{n+n^2\hat{\tau}^2})}} \Bigg)\\
& = \Phi \Bigg(\Phi^{-1}(1-\alpha) - \frac{(n+ \hat{\tau}^2n^2)d+ n^2 \lVert\vl_1\rVert^2 - (n+ \hat{\tau}^2n^2-1)d- n^2 \lVert\vl_0\rVert^2}{(n + \hat{\tau}^2n^2) \sqrt{2d+ 4  \frac{n^2}{n+ \hat{\tau}^2n^2}  \lVert\vl_1\rVert^2)}} \Bigg)\\
& =\Phi \Bigg(\Phi^{-1}(1-\alpha) - \frac{d + n^2(\lVert\vl_1\rVert^2 - \lVert\vl_0\rVert^2)}{(n + \hat{\tau}^2n^2) \sqrt{2d+ 4  \frac{n^2}{n+ \hat{\tau}^2n^2}  \lVert\vl_1\rVert^2)}} \Bigg)\\
& =\Phi \Bigg(\Phi^{-1}(1-\alpha) - \frac{d + n^2(\lVert\vl_1\rVert^2 - \lVert\vl_0\rVert^2)}{(n + \hat{\tau}^2n^2) \sqrt{2d+ 4  \frac{n^2}{n+ \hat{\tau}^2n^2}  \lVert\vl_1\rVert^2)}} \Bigg).
\end{align}
We further obtain $\lVert\vl_0\rVert^2 = \frac{(n-1+\hat{\tau}^2n^2)^2}{n^2}K$ and $\lVert\vl_1\rVert^2 = \frac{(n+\hat{\tau}^2n^2)}{n^2}K$ where $K=\lVert \tilde{\vmu} - \tilde{\vx}' \rVert^2 $. Therefore, we have
\begin{align}
\beta \approx & \Phi \Bigg(\Phi^{-1}(1-\alpha) - \frac{d + (2(n+\hat{\tau}^2n^2)-1)K}{(n+ \hat{\tau}^2n^2)\sqrt{2d+ 4(n+\hat{\tau}^2n^2)K}} \Bigg).
\end{align}
Thus, for large $n,d$ the trade-off curve can be well approximated by the $\mu$-GMIP trade-off with:
\begin{align}
\mu = \frac{d + \big(2(n+\hat{\tau}^2n^2) - 1\big) K }{(n+\hat{\tau}^2n^2)\sqrt{2d + 4(n+\hat{\tau}^2n^2) K }}.
\label{eqn:finalmuneffective}
\end{align}
We observe that the hardness of the test decreases with $K$. Therefore, the stochastic composition of tests is uniformly at least as hard as the test with the largest $K$ possible (\Cref{app_thm:worstcasestoasticcomp}). We obtain the result shown in the paper by replacing the data-dependent noise $\hat{\tau}$ by data-independent noise $\tau$ at a ratio of $\hat{\tau}^2 = \tau^2/C^2$ as detailed in the next section. 
\end{proof}

\subsection{Proof of Theorem \ref{theorem:onestep_dp_sgd} and Corollary \ref{corollary:onestep_dp_sgd} with data independent noise}
\label{sec:app_onestep_dp_sgdunitnoise}
In the previous section, we have assumed that $Y \sim \hat{\tau}^2\mSigma$.
This assumption does not quite match common practice; in practice, the Gaussian mechanism adds noise of the form: $Y\sim \tau^2\mI$.
Hence, the following subsection investigates the effect of adding unit noise $Y\sim \tau^2\mI$.

\begin{proof}
We first study the case of data-dependent noise using an eigenspace transform.
Denoting an eigenvalue composition $\mSigma=\mQ\mLambda\mQ^\top$ and performing the corresponding transform (multiplication by $\mQ^\top$) in \Cref{eqn:addsigmanoise}, we obtain: 
\begin{align}
& \Longleftrightarrow \text{Test}\left[\mathcal{N}\left(\frac{n-1}{n}\left(\vmu-\vx^\prime\right), \frac{(n-1)}{n^2}\mSigma + \hat{\tau}^2 \mSigma \right), \mathcal{N}\left(\vmu -\vx^\prime, \frac{1}{n}\mSigma + \hat{\tau}^2 \mSigma \right)\right]\\
& \Longleftrightarrow \text{Test}\left[\mathcal{N}\left(-\frac{1}{n}\mQ^{\top}\left(\vmu-\vx^\prime\right), \bigg(\frac{(n-1)}{n^2} + \hat{\tau}^2 
\bigg) \mQ^{\top}\mSigma\mQ\right), \mathcal{N}\left(\mathbf{0}, \bigg(\frac{1}{n} + \hat{\tau}^2 \bigg) \mQ^{\top}\mSigma\mQ\right)\right]\\
& \Longleftrightarrow \text{Test}\left[\mathcal{N}\left(-\frac{1}{n}\mQ^{\top}\left(\vmu-\vx^\prime\right), \bigg(\frac{(n-1)}{n^2} + \hat{\tau}^2 
\bigg) \mLambda\right), \mathcal{N}\left(\mathbf{0}, \bigg(\frac{1}{n} + \hat{\tau}^2 \bigg) \mLambda\right)\right].
\end{align}
We see that the test decomposes in a series of $d$ dimension-wise 1D hypotheses tests of the form:
\begin{align}
\text{Test}\left[\mathcal{N} \left(-\frac{1}{n}\mQ^\top\left(\vmu-\vx^\prime\right)_i, \frac{n-1}{n^2}\lambda_i + \hat{\tau}^2\lambda_i\right), \mathcal{N}\left(0,\frac{1}{n}\lambda_i+ \hat{\tau}^2\lambda_i\right)\right].
\end{align}
Here, $\lambda_i$ are the eigenvalues of $\mSigma$ that are on the diagonal of the matrix $\mLambda$.
On the other hand,  when adding unit noise of magnitude $\tau$, i.e., setting $Y=\mathcal{N}(\mathbf{0}, \tau^2 \mI)$ in \Cref{eqn:addsigmanoise}:
\begin{align}
&\text{Test}\left[\mathcal{N}\left(\frac{n-1}{n}\left(\vmu-\vx^\prime\right), \frac{(n-1)}{n^2}\mSigma + \tau^2 \mI\right), \mathcal{N}\left(\vmu -\vx^\prime, \frac{1}{n}\mSigma + \tau^2 \mI \right)\right]\\
& \Longleftrightarrow \text{Test}\left[\mathcal{N}\left(-\frac{1}{n}\mQ^\top\left(\vmu-\vx^\prime\right), \bigg(\frac{(n-1)}{n^2}\mQ^\top\mSigma\mQ + \tau^2\mQ^\top\mI\mQ
\bigg)\right),\right. \\
&\left.\mathcal{N}\left(\mathbf{0}, \bigg(\frac{1}{n}\mQ^\top\mSigma\mQ + \tau^2\mQ^\top\mI\mQ\bigg) \right)\right] \\
& \Longleftrightarrow 
\text{Test}\left[\mathcal{N}\left(-\frac{1}{n}\mQ^\top\left(\vmu-\vx^\prime\right), \frac{(n-1)}{n^2}\mLambda + \tau^2\mI
\right), \mathcal{N}\left(\mathbf{0},\frac{1}{n}\mLambda + \tau^2\mI \right)\right].
\label{eqn:qtransformtest}
\end{align}
This test also decomposes in $d$ 1D tests of the form:
\begin{align}
\text{Test}\left[\mathcal{N}\left(-\underset{\mu_i}{\underbrace{\frac{1}{n}\mQ^\top\left(\vmu-\vx^\prime\right)_i}}, \underset{\sigma_1^2}{\underbrace{\frac{n-1}{n^2}\lambda_i + \tau^2}}\right), \mathcal{N}\left(0,\underset{\sigma_2^2}{\underbrace{\frac{1}{n}\lambda_i+ \tau^2}}\right)\right].
\end{align}


The test is therefore equally hard as testing $d$ independent normal random variables because the covariance matrices shown are diagonal. We will show that by setting
\begin{align}
\tau^2 = \hat{\tau}^2 \cdot \max_i\{\lambda_i\}
\end{align}
each individual, dimension-wise test is made strictly harder than the corresponding dimension-wise test in the data-dependent noise case and therefore the composed test is also harder (this is shown in \Cref{app_lem:dimensioncomposition}).

We will do so by computing the power function of the test and showing that is is montonically decreasing in $\tau$, which means that for each individual test, that higher effective noise in that dimension makes the test harder.
To show this, note that the likelihood ratio for this test results in 
\begin{align}
S = \frac{(m-\hat{\mu})^2}
{\hat{\sigma}^{2}}
\end{align}
where 
\begin{align}
\hat{\sigma}^{2} =\bigg(\frac{\sigma_2^2-\sigma_1^2}{\sigma_1^2\sigma_2^2}\bigg)^{-1} = \frac{\left(\frac{n-1}{n^2}\lambda_i +\tau^2 \right)( \frac{1}{n} \lambda_i+ \tau^2 ) n^2}{\lambda_i},
\end{align}
\begin{align}
\hat{\mu} &= \hat{\sigma}^{2} \bigg( \frac{1}{\sigma_1^2} \vmu_1 - \frac{1}{\sigma_2^2} \vmu_2 \bigg) \\
&= \hat{\sigma}^{2} \frac{1}{\sigma_2^2 \sigma_1^2} \big( \sigma_2^2 \vmu_1 - \sigma_1^2 \vmu_2 \big) \\
&= \frac{\sigma_2^2 \sigma_1^2 }{\sigma_2^2 - \sigma_1^2} \frac{1}{\sigma_2^2 \sigma_1^2} \big( \sigma_2^2 \vmu_1 - \sigma_1^2 \vmu_2 \big) \\
&= \frac{1}{\sigma_2^2 -\sigma_1^2} \big( \sigma_2^2 \vmu_1 - \sigma_1^2 \vmu_2 \big) \\
&= \frac{n^2}{\lambda_i} \left(-\frac{1}{n}\mQ^\top\left(\vmu-\vx^\prime\right)\left(\frac{1}{n}\lambda_i+ \tau^2\right)  \right)\\
&= -\mQ^\top\left(\vmu-\vx^\prime\right)_i \left(\frac{n^2}{\lambda_i}\left(\frac{\lambda_i}{n^2} +\frac{\tau^2}{n}\right)\right)\\
&= -\mQ^\top\left(\vmu-\vx^\prime\right)_i \left(1+ \frac{n\tau^2}{\lambda_i}\right) = n\mu_i \left(1+ \frac{n\tau^2}{\lambda_i}\right).
\end{align}
\textbf{Distribution under the null hypothesis}: Again, as before we derive the distribution under the null hypothesis:
\begin{align}
 m-\hat{\mu} &\sim \mathcal{N}\left(\mu_i\left(1-n\left(1+ \frac{n\tau^2}{\lambda_i}\right)\right),\frac{(n-1)}{n^2}\lambda_i + \tau^2 \right)\\
  \frac{m-\hat{\mu}}{\sigma_1} &\sim \mathcal{N}\left(\mu_i\frac{1-n- \frac{n^2\tau^2}{\lambda_i}}{\sqrt{\frac{(n-1)}{n^2}\lambda_i + \tau^2}}, 1\right)\\
  \frac{m-\hat{\mu}}{\sigma_1} &\sim \mathcal{N}\left(-\mu_i\frac{n -1 + \frac{n^2\tau^2}{\lambda_i}}{\sqrt{\frac{\lambda_i}{n^2}\left(n-1 + \frac{n^2\tau^2}{\lambda_i}\right)}}, 1\right)\\
  \frac{m-\hat{\mu}}{\sigma_1} &\sim \mathcal{N}\left(-\mu_i\sqrt{n-1 + \frac{n^2\tau^2}{\lambda_i}}{\sqrt{\frac{n^2}{\lambda_i}}}, 1\right)\\
  \frac{m-\hat{\mu}}{\sigma_1} &\sim \mathcal{N}\left(-\mu_i\left(n\sqrt{\frac{n-1}{\lambda_i}+\frac{n^2\tau^2}{\lambda_i^2}}\right), 1\right)\\
  \left(\frac{m-\hat{\mu}}{\sigma_1}\right)^2 &\sim \chi'^2_1\left(\mu_i^2 n^2\left(\frac{n-1}{\lambda_i}+\frac{n^2\tau^2}{\lambda_i^2}\right)\right) = \chi'^2_1\left(\gamma_0\right),
\end{align}
where $\chi_d(\gamma)$ again denotes the non-central $\chi$-distribution with $d$ degrees of freedom.

\textbf{Rejection region:} Therefore, the rejection region of the null hypothesis at a significance level of $\alpha$ can be formulated as:
\begin{align}
\left\{ \frac{\hat{\sigma}^2}{\sigma_1^2} S \geq \text{CDF}_{\mathcal{X}'^2_1(\gamma_0)}^{-1}(1-\alpha)\right\} =
 \left\{S \geq \frac{\sigma_1^2} {\hat{\sigma}^2}\text{CDF}_{\mathcal{X}'^2_1(\gamma_0)}^{-1}(1-\alpha)\right\},
\end{align}
where $\frac{\sigma_1^2} {\hat{\sigma}^2} = \frac{(\frac{n-1}{n} \lambda_i + \tau^2) \lambda_i}{n^2 (\frac{n-1}{n} \lambda_i + \tau^2) (\frac{1}{n} \lambda_i + \tau^2)} = \frac{\lambda_i}{n^2(\frac{1}{n}\lambda_i+\tau^2)} = \frac{\lambda_i}{n\lambda_i+n^2\tau^2}$.

\textbf{Distribution under the alternative hypothesis}: As before, we also need to derive the distribution of the test statistic under the alternative hypothesis:
\begin{align}
 m-\hat{\mu} &\sim \mathcal{N}\left(-n\mu_i \left(1+ \frac{n\tau^2}{\lambda_i}\right),\frac{1}{n}\lambda_i + \tau^2 \right)\\
\frac{m-\hat{\mu}}{\sigma_2} &\sim \mathcal{N}\left(-n\mu_i\frac{1+ \frac{n\tau^2}{\lambda_i}}{\sqrt{\frac{1}{n}\lambda_i + \tau^2}}, 1\right)\\
\frac{m-\hat{\mu}}{\sigma_2} &\sim \mathcal{N}\left(-n\mu_i\frac{1+ \frac{n\tau^2}{\lambda_i}}{\sqrt{\frac{\lambda_i}{n}\left(1 + \frac{n\tau^2}{\lambda_i}\right)}}, 1\right)\\
\frac{m-\hat{\mu}}{\sigma_2} &\sim \mathcal{N}\left(-\mu_i n\sqrt{\frac{n}{\lambda_i}\left(1 + \frac{n\tau^2}{\lambda_i}\right)}, 1\right)\\
\left(\frac{m-\hat{\mu}}{\sigma_2}\right)^2 &\sim \mathcal{X}'^2_1\left(\mu_i^2n^2 \left(\frac{n}{\lambda_i} + \frac{n^2\tau^2}{\lambda_i^2}\right)\right)
\end{align}
Therefore, 
\begin{align}
\frac{\hat{\sigma}^2}{\sigma_2^2}S\sim \mathcal{X}'^2_1\left(\mu_i^2n^2 \left(\frac{n}{\lambda_i} + \frac{n^2\tau^2}{\lambda_i^2}\right)\right) = \mathcal{X}'^2_1\left(\gamma_1\right).
\end{align}
\textbf{False negative rate:} Finally, 1 - power of the test is given by:
\begin{align}
 \beta(\alpha) &= P\left\{S \leq \frac{\sigma_1^2} {\hat{\sigma}^2}\text{CDF}_{\mathcal{X}'^2_1(\gamma_0)}^{-1}(1-\alpha)\right\}\\
 &=P\left\{\frac{\hat{\sigma}^2}{\sigma_2^2}S \leq \frac{\hat{\sigma}^2}{\sigma_2^2}\frac{\sigma_1^2} {\hat{\sigma}^2}\text{CDF}_{\mathcal{X}'^2_1(\gamma_0)}^{-1}(1-\alpha)\right\}\\
&=\text{CDF}_{\mathcal{X}'^2_1(\gamma_1)}\left(\frac{\sigma_1^2} {\sigma_2^2}\text{CDF}_{\mathcal{X}'^2_1(\gamma_0)}^{-1}(1-\alpha)\right).
\end{align}
We will now introduce two quantities on which the power of this test depends. In particular
\begin{align}
q \coloneqq \frac{n +\frac{n^2\tau^2}{\lambda_i}-1}{n + \frac{n^2\tau^2}{\lambda_i}},
\end{align}
which has the intriguing property that 
\begin{align}
\frac{\gamma_0}{\gamma_1} = \frac{\sigma_1^2}{\sigma_2^2} = q.
\end{align}
Using this insight, the trade-off function can be expressed as 
\begin{align}
\beta(\alpha, \gamma_1(\lambda_i,\mu_i), q(\lambda_i))=\text{CDF}_{\mathcal{X}'^2_1(\gamma_1)}\left(q\text{CDF}_{\mathcal{X}'^2_1(q\gamma_1)}^{-1}(1-\alpha)\right).
\end{align}
We determine the hardest test by showing that
when considering a fixed $\mu_i$, $\alpha$, we obtain
\begin{align}
\frac{\partial \beta}{\partial (\tau^2)} > 0,
\end{align}
which indicates that at a higher level of noise, the type 2 error rate will increase, making the test harder (more privacy). The derivative computation can be found in \Cref{sec:derivtest}. When we set $\tau^2 = \hat{\tau}^2 \cdot \max_i\{\lambda_i\}$, the individual tests in the case of data-independent noise are thus harder ($\mu_i$, $n$, $d$ stays the same) than the respective dimension-wise test for the data-dependent noise. Therefore the corresponding bound derived for $\hat{\tau}$ will hold for data independent noise of strength $\tau$ as well.

The largest possible eigenvalue $\lambda_i$ of the covariance matrix $\mSigma$ when cropping each vector at norm $C$ is given by $C^2$. Therefore, we can set $\tau^2 = \hat{\tau}^2 C^2$ to obtain a harder test with data independent noise. On the converse, applying data-independent noise of level $\tau^2$ results in a harder test than applying data-dependent noise with $\hat{\tau}^2=\frac{\tau^2}{C^2}$. Plugging this result in \Cref{eqn:finalmuneffective} requires to only replace $\hat{\tau}^2=\frac{\tau^2}{C^2}$ and introducing the quantity $n_{\text{effective}}$ as 
\begin{align}
    n_{\text{effective}} = n+\frac{n^2\tau^2}{C^2}.
\end{align}
\end{proof}

\subsubsection{Derivatives of the type 2 error}
\label{sec:derivtest}
In this section we calculate the derivatives of the type 2 error function
\begin{align}
\beta(\alpha, \gamma_1(\lambda_i,\mu_i), q(\lambda_i) =\text{CDF}_{\mathcal{X}'^2_1(\gamma_1)}\left(q\text{CDF}_{\mathcal{X}'^2_1(q\gamma_1)}^{-1}(1-\alpha)\right).
\end{align}


Considering a fixed $\mu_i$ we show that
\begin{align}
    \frac{\partial \beta}{\partial(\tau^2)} > 0,
\end{align}
which indicates that with larger added noise, the type 2 error rate will decrease, making the test uniformly harder (more privacy). Thus the hardest test is the one with maximum $\tau^2$. To emphasize that we are deriving by $(\tau^2)$ (the squared quantity), we write $\tau_2 \coloneqq \tau^2$.

We perform the following calculations:

\textbf{Derivatives of CDFs and inverse CDFs.} We will start by deriving some useful derivatives of the CDFs and inverse CDFs. We start with the inverse CDF.
Initially, Let $u = \sqrt{\text{CDF}_{\mathcal{X}'^2_1(\gamma_0)}^{-1}(1-\alpha)}$ or $u^2 = \text{CDF}_{\mathcal{X}'^2_1(\gamma_0)}^{-1}(1-\alpha)$ (note that the quantile function here is always $>0$):
\begin{align}
\Phi(u-\sqrt{\gamma_0})-\Phi(-(u+\sqrt{\gamma_0}))=\alpha.
\end{align}
Keeping $\alpha$ constant, we can use the implicit function derivative theorem
with $F(\gamma_0, u(\gamma_0)) = \Phi(u-\sqrt{\gamma_0})-\Phi(-(u+\sqrt{\gamma_0)})-\alpha = 0$, we obtain
\begin{align}
\frac{\partial u}{\partial\gamma_0} & = -\frac{-\phi(u-\sqrt{\gamma_0)})\frac{1}{2\sqrt{\gamma}}+\phi(-u-\sqrt{\gamma_0)})\frac{1}{2\sqrt{\gamma}}}{\phi(u-\sqrt{\gamma_0)})+\phi(-u-\sqrt{\gamma_0)}} \\
&= \frac{1}{2\sqrt{\gamma_0}}\frac{\phi(u-\sqrt{\gamma_0)})-\phi(-u-\sqrt{\gamma_0)})}{\phi(u-\sqrt{\gamma_0)})+\phi(-u-\sqrt{\gamma_0)})}
\end{align}
We furthermore have
\begin{align}
\frac{\text{CDF}_{\mathcal{X}'^2_1(\gamma_0)}^{-1}(1-\alpha)}{\partial \gamma_0} = 2 \sqrt{\text{CDF}_{\mathcal{X}'^2_1(\gamma_0)}^{-1}(1-\alpha)} \frac{\partial u}{\partial\gamma_0}.
\end{align}

We note that 
\begin{align}\text{CDF}_{\mathcal{X}'^2_1(\gamma_1)}(s) =
\Phi(\sqrt{s}-\sqrt{\gamma_1})-\Phi(-\sqrt{s}-\sqrt{\gamma_1})),
\end{align}
where $s$ is the argument of the CDF.
The derivative 
\begin{align}
\frac{\partial\text{CDF}_{\mathcal{X}'^2_1(\gamma_1)}}{\partial \gamma_1} = -\frac{1}{2\sqrt{\gamma_1}}\left(\phi(\sqrt{s}-\sqrt{\gamma_1})-\phi(-\sqrt{s}-\sqrt{\gamma_1})\right) \leq 0
\end{align}
\begin{align}
\frac{\partial\text{CDF}_{\mathcal{X}'^2_1(\gamma_1)}(s)}{\partial s} = \frac{1}{2\sqrt{s}}\left(\phi(\sqrt{s}-\sqrt{\gamma_1})+\phi(-\sqrt{s}-\sqrt{\gamma_1})\right) > 0.
\end{align}

\textbf{Derivative with respect to $\mu_i$.} We compute the derivative with respect to $\mu_i$ first, which will be handy in the subsequent derivation of the derivative by $\tau_2$. To denote that we consider a fixed $\alpha$, we write $\beta_\alpha$.
We have 
\begin{align}
\beta_\alpha(\gamma_1(\mu_i, \lambda_i)), q(\lambda_i))=\text{CDF}_{\mathcal{X}'^2_1(\gamma_1)}\left(q\text{CDF}_{\mathcal{X}'^2_1(q\gamma_1)}^{-1}(1-\alpha)\right).
 \end{align}
Thus, we can calculate
 \begin{align}
     \frac{\partial \beta_\alpha}{\partial \mu_i}=\frac{\partial \beta_\alpha}{\partial \gamma_1}\frac{\partial \gamma_1}{\partial \mu_i}= \underset{<0?}{\underbrace{{\left(\frac{\partial\text{CDF}_{\mathcal{X}'^2_1(\gamma_1)}(s)}{\partial \gamma_1}+\frac{\partial\text{CDF}_{\mathcal{X}'^2_1(\gamma_1)}(s)}{\partial s} \frac{\partial s}{\partial \gamma_1}\right)}}}\underset{>0}{\underbrace{\frac{\partial \gamma_1}{\partial \mu_i}}},
 \end{align}
where $s =q\text{CDF}_{\mathcal{X}'^2_1(q\gamma_1)}^{-1}(1-\alpha)$. We can confirm that $\frac{\partial \gamma_1}{\partial \mu_i} > 0$ because
\begin{align}
\gamma_1 = \mu_i^2n^2 \left(\frac{n}{\lambda_i} + \frac{n^2\tau^2}{\lambda_i^2}\right)
\end{align}
 is monotonically increasing in $\mu_i$, because all eigenvalues $\lambda_i > 0$ and the factor after $\mu_i^2$ is positive overall. To establish that the overall derivative is negative, we therefore have to confirm that the first factor is always negative: 
\begin{align}
    \frac{\partial\text{CDF}_{\mathcal{X}'^2_1(\gamma_1)}(s)}{\partial \gamma_1} &= -\frac{1}{2\sqrt{\gamma_1}}\left(\phi(\sqrt{s}-\sqrt{\gamma_1})-\phi(-\sqrt{s}-\sqrt{\gamma_1})\right)\\
    \frac{\partial s}{\partial \gamma_1} &= q  \frac{\text{CDF}_{\mathcal{X}'^2_1(q\gamma_1)}^{-1}(1-\alpha)}{\partial \gamma_1} = 2q^2\sqrt{\text{CDF}_{\mathcal{X}'^2_1(q\gamma_1)}^{-1}(1-\alpha)} \frac{\partial u}{\partial\gamma_0}\bigg|_{\gamma_0 = q\gamma_1}.
\end{align}

Let $s=q\text{CDF}_{\mathcal{X}'^2_1(qL)}^{-1}(1-\alpha)$
\begin{align}
    \frac{\partial\text{CDF}_{\mathcal{X}'^2_1(\gamma_1)}(s)}{\partial s} \frac{\partial s}{\partial \gamma_1} &= \frac{\text{CDF}_{\mathcal{X}'^2_1(\gamma_1)}(s)}{\partial s} 2q^2\sqrt{\text{CDF}_{\mathcal{X}'^2_1(q\gamma_1)}^{-1}(1-\alpha)} \frac{\partial u}{\partial\gamma_0}|_{\gamma_0 = q\gamma_1}\\
    &= \frac{\text{CDF}_{\mathcal{X}'^2_1(\gamma_1)}(s)}{\partial s} 2q^2\sqrt{\frac{s}{q}}\frac{\partial u}{\partial\gamma_0}|_{\gamma_0 = q\gamma_1}\\
    &= \frac{\text{CDF}_{\mathcal{X}'^2_1(\gamma_1)}(s)}{\partial s} 2q^2\sqrt{\frac{s}{q}} \frac{1}{2\sqrt{\gamma_0}}\frac{\phi(u-\sqrt{\gamma_0)})-\phi(-u-\sqrt{\gamma_0)})}{\phi(u-\sqrt{\gamma_0)})+\phi(-u-\sqrt{\gamma_0)})}.
\end{align}
We can plug in $\gamma_0 = \gamma_1q$, and $u = \frac{\sqrt{s}}{\sqrt{q}}$
\begin{align}
    &= \frac{\text{CDF}_{\mathcal{X}'^2_1(\gamma_1)}(s)}{\partial s} 2q^2\sqrt{\frac{s}{q}} \frac{1}{2\sqrt{\gamma_1q}}\frac{\phi(\frac{\sqrt{s}}{\sqrt{q}}-\sqrt{\gamma_1q)})-\phi(-\frac{\sqrt{s}}{\sqrt{q}}-\sqrt{\gamma_1q)})}{\phi(\frac{\sqrt{s}}{\sqrt{q}}-\sqrt{\gamma_1q)})+\phi(-\frac{\sqrt{s}}{\sqrt{q}}-\sqrt{\gamma_1q)})}\\
    &= \frac{1}{2\sqrt{s}}\left(\phi(\sqrt{s}-\sqrt{\gamma_1})+\phi(-\sqrt{s}-\sqrt{\gamma_1})\right) \frac{q\sqrt{s}}{\sqrt{\gamma_1}}\frac{\phi(\frac{\sqrt{s}}{\sqrt{q}}-\sqrt{\gamma_1q)})-\phi(-\frac{\sqrt{s}}{\sqrt{q}}-\sqrt{\gamma_1q)})}{\phi(\frac{\sqrt{s}}{\sqrt{q}}-\sqrt{\gamma_1q)})+\phi(-\frac{\sqrt{s}}{\sqrt{q}}-\sqrt{\gamma_1q)})}\\
    &= \frac{1}{2}\frac{q}{\sqrt{\gamma_1}}\left(\phi(\sqrt{s}-\sqrt{\gamma_1})+\phi(-\sqrt{s}-\sqrt{\gamma_1})\right) \frac{\phi(\frac{\sqrt{s}}{\sqrt{q}}-\sqrt{\gamma_1q)})-\phi(-\frac{\sqrt{s}}{\sqrt{q}}-\sqrt{\gamma_1q)})}{\phi(\frac{\sqrt{s}}{\sqrt{q}}-\sqrt{\gamma_1q})+\phi(-\frac{\sqrt{s}}{\sqrt{q}}-\sqrt{\gamma_1q})}
\end{align}
In total we obtain
\begin{align}
    \frac{\partial \beta_\alpha(\gamma_1)}{\partial \gamma_1} &= -\frac{1}{2\sqrt{\gamma_1}}\left(\phi(\sqrt{s}-\sqrt{\gamma_1})-\phi(-\sqrt{s}-\sqrt{\gamma_1})\right)\\
    &+\frac{1}{2}\frac{q}{\sqrt{\gamma_1}}\left(\phi(\sqrt{s}-\sqrt{\gamma_1})+\phi(-\sqrt{s}-\sqrt{\gamma_1})\right) \frac{\phi(\frac{\sqrt{s}}{\sqrt{q}}-\sqrt{\gamma_1q)})-\phi(-\frac{\sqrt{s}}{\sqrt{q}}-\sqrt{\gamma_1q)})}{\phi(\frac{\sqrt{s}}{\sqrt{q}}-\sqrt{\gamma_1q})+\phi(-\frac{\sqrt{s}}{\sqrt{q}}-\sqrt{\gamma_1q})}\\
    &=\frac{1}{2\sqrt{\gamma_1}}\left(-S_1 +qS_2\frac{T_1}{T_2}\right). 
\end{align}
By inspection of the terms, we see that
\begin{align}
    S_1 = \phi(\sqrt{s}-\sqrt{\gamma_1})-\phi(-\sqrt{s}-\sqrt{\gamma_1}) \geq 0
\end{align}
because $|\sqrt{s}-\sqrt{\gamma_1}| \leq |-\sqrt{s}-\sqrt{\gamma_1}|$ and the normal density $\phi$ decreases with the absolute value of its argument.
We can therefore rewrite the expression as
\begin{align}
\frac{\partial \beta_\alpha(\gamma_1)}{\partial \gamma_1} = \frac{S_1}{2\sqrt{\gamma_1}}\left(-1 + q\frac{S_2 T_1}{S_1 T_2}\right)
\end{align}
and the sign is determined by the second factor (the first fraction $\frac{S_1}{2\sqrt{\gamma_1}}$ will be non-negative).
\begin{align}
    \frac{S_2 T_1}{S_1 T_2}&=\frac{\left(\phi(\sqrt{s}-\sqrt{\gamma_1})+\phi(-\sqrt{s}-\sqrt{\gamma_1})\right)}{\left(\phi(\sqrt{s}-\sqrt{\gamma_1})-\phi(-\sqrt{s}-\sqrt{\gamma_1})\right)}\frac{\left(\phi(\frac{\sqrt{s}}{\sqrt{q}}-\sqrt{\gamma_1q)})-\phi(-\frac{\sqrt{s}}{\sqrt{q}}-\sqrt{\gamma_1q)})\right)}{\left(\phi(\frac{\sqrt{s}}{\sqrt{q}}-\sqrt{\gamma_1q})+\phi(-\frac{\sqrt{s}}{\sqrt{q}}-\sqrt{\gamma_1q})\right)}\\
&=\frac{\Gamma -\left(\phi(\sqrt{s}-\sqrt{\gamma_1})\phi(-\frac{\sqrt{s}}{\sqrt{q}}-\sqrt{\gamma_1q})-\phi(-\sqrt{s}-\sqrt{\gamma_1})\phi(\frac{\sqrt{s}}{\sqrt{q}}-\sqrt{\gamma_1q})\right)}{\Gamma + \left(\phi(\sqrt{s}-\sqrt{\gamma_1})\phi(-\frac{\sqrt{s}}{\sqrt{q}}-\sqrt{\gamma_1q})-\phi(-\sqrt{s}-\sqrt{\gamma_1})\phi(\frac{\sqrt{s}}{\sqrt{q}}-\sqrt{\gamma_1q})\right)}\\
&=\frac{\Gamma-\Delta}{\Gamma+\Delta},
\end{align}
where we have:
\begin{align}
    \Delta &= \frac{1}{2\pi}\left(\exp\left(-\frac{(\sqrt{s}-\sqrt{\gamma_1})^2+ (-\frac{\sqrt{s}}{\sqrt{q}}-\sqrt{\gamma_1q})^2}{2}\right)-\exp\left(-\frac{(-\sqrt{s}-\sqrt{\gamma_1})^2+ (\frac{\sqrt{s}}{\sqrt{q}}-\sqrt{\gamma_1q})^2}{2}\right)\right)\\
    &=\frac{1}{2\pi}\left(\exp\left(-\frac{s-2\sqrt{s}\sqrt{\gamma_1}+\gamma_1 + \frac{s}{q} + 2\frac{\sqrt{s}\sqrt{\gamma_1q}}{\sqrt{q}}+\gamma_1q}{2}\right)\right.\\
    &\left.-\exp\left(-\frac{s+2\sqrt{s}\sqrt{\gamma_1}+\gamma_1 + \frac{s}{q} - 2\frac{\sqrt{s}\sqrt{\gamma_1q}}{\sqrt{q}}+\gamma_1q}{2}\right)\right)\\
    &=\frac{1}{2\pi}\left(\exp\left(-\frac{s+\gamma_1 + \frac{s}{q} + \gamma_1q}{2}\right)-\exp\left(-\frac{s+\gamma_1 + \frac{s}{q} + \gamma_1q}{2}\right)\right)\\
    &=0.
\end{align}
Bringing this result ($\frac{S_2 T_1}{S_1 T_2}=1$) back, we obtain 
\begin{align}
\frac{\partial \beta_\alpha(\gamma_1)}{\partial \gamma_1} = -\frac{S_1}{2\sqrt{\gamma_1}}\left(1-q\right) \leq 0.
\end{align}

 \textbf{Derivative w.r.t. $\tau^2$}.
 We denote $\tau^2=\tau_2$ to emphasize that we are deriving with respect to the square of $\tau$:
  \begin{align}
     \frac{\partial \beta_\alpha}{\partial \tau_2}=\underset{<0}{\underbrace{\frac{\partial \beta_\alpha}{\partial \gamma_1}}}\underset{>0}{\underbrace{\frac{\partial \gamma_1}{\partial \tau_2}}}+ \underset{>0}{\underbrace{\frac{\partial \beta_\alpha}{\partial q}}}\underset{>0}{\underbrace{\frac{\partial q}{\partial \tau_2}}}.
 \end{align}

 We perform the following calculations:
 \begin{align}
 \frac{\partial \beta_\alpha}{\partial \gamma_1} &= -\frac{S_1}{2\sqrt{\gamma_1}}\left(1-q\right) <0\\ 
\frac{\partial \gamma_1}{\partial \tau_2} &= \frac{n^4 \mu^2}{\lambda_i^2} >0 \\
\frac{\partial \beta_\alpha}{\partial q} &= \frac{S_2\sqrt{s}}{2q}+\frac{S_1\sqrt{\gamma_1}}{2} >0\\
\frac{\partial q}{\partial \tau_2} &= \frac{n^2}{\lambda_i(n+r)^2} >0,
 \end{align}
 where $r=\frac{n^2\tau_2}{\lambda_i}$.
 Plugging the terms together, we see that 
 \begin{align}
     \frac{\partial \beta_\alpha}{\partial \tau_2}&=  -\frac{S_1}{2\sqrt{\gamma_1}}\left(1-q\right)\frac{n^4 \mu^2}{\lambda_i^2}+\left(\frac{S_2\sqrt{s}}{2q}+\frac{S_1\sqrt{\gamma_1}}{2}\right)\frac{n^2}{\lambda_i(n+r)^2}\\
     &= \frac{n^2}{2\lambda_i}\left(-\frac{S_1}{\sqrt{\gamma_1}}\left(1-q\right)\frac{n^2 \mu^2}{\lambda_i}+\left(\frac{S_2\sqrt{s}}{q}+S_1\sqrt{\gamma_1}\right)\frac{1}{(n+r)^2}\right),
 \end{align}
 using $1-q = \frac{1}{n + \frac{n^2\tau^2}{\lambda_i}} = \frac{1}{n+r}$ we obtain
 \begin{align}
     \frac{\partial \beta_\alpha}{\partial \tau_2}&= \frac{n^2}{2\lambda_i(n+r)}\left(-\frac{S_1}{\sqrt{\gamma_1}}\frac{n^2 \mu^2}{\lambda_i}+\left(\frac{S_2\sqrt{s}}{q}+S_1\sqrt{\gamma_1}\right)\frac{1}{(n+r)}\right).
 \end{align}
Further using $\frac{n^2 \mu^2}{\lambda_i^2}=  \frac{\gamma_1}{n+r}$ we obtain
 \begin{align}
\frac{\partial \beta_\alpha}{\partial \tau_2} &= \frac{n^2}{2\lambda_i(n+r)^2}\left(-S_1\sqrt{\gamma_1}+\frac{S_2\sqrt{s}}{q}+S_1\sqrt{\gamma_1}\right)\\
&= \frac{S_2n^2\sqrt{s}}{2q\lambda_i(n+r)^2} > 0.
 \end{align}

\section{On the Relation between $f$-MIP and $f$-DP}
\subsection{Proof of \Cref{thm:mipweakerthandp}}
\label{sec:proofmipweakerdp}
We first restate the theorem:

\textbf{Theorem 4.2}~($f$-DP implies $f$-MIP). \textit{
Let an algorithm $\mathcal{A}: D^{n} \rightarrow \mathbb{R}^d$ be $f$-differentially private \cite{dong2022gaussian}. Then, the algorithm $\mathcal{A}$ will also be $f$-membership inference private.}

\begin{proof} 
If the (probabilistic) algorithm $\mathcal{A}$ is $f$-DP, this requires that for all neighboring datasets $S$ and $S^\prime$ (where only a single instance is changed), we have \citep{dong2022gaussian}
\begin{align}
    \text{Test}\left(A(S), A(S^\prime)\right) \geq f.
\end{align}
This includes any randomly sampled dataset $\mathcal{X} \in D^{n-1}$ plus an instance $\bm{x}$ as well as the dataset $\mathcal{X}$ plus the additional instance $\bm{x}^\prime$ that is known to the attacker. Therefore setting
\begin{align}
    S=\mathcal{X} \cup \left\{\bm{x}\right\}\\
    S^\prime=\mathcal{X} \cup \left\{\bm{x}^\prime\right\}
\end{align}
we immediately see that for every $\mathcal{X}$, $\bm{x}$, $\bm{x}^\prime$, under $f$-DP, we have
\begin{align}
    \text{Test}\left(\mathcal{A}(\mathcal{X} \cup \left\{\bm{x}\right\}), \mathcal{A}(\mathcal{X} \cup \left\{\bm{x}^\prime\right\})\right) \geq f.
\label{eq:proof52dp}\end{align}
Consulting our hypothesis test formulation of MI attacks, we note that the sets $S, S^\prime$ correspond to the inputs to $A$ which result in output distributions $A_0$ and $A_1(\bm{x}^\prime)$ in the formulation of \Cref{eq:hypothesisformulation}, but for a specific $\bm{x}$, $\mathcal{X}$. 

Additionally, as opposed to DP, in the membership inference attack scenario, the attacker has no knowledge of $\bm{x}$, $\mathcal{X}$. As the privacy constraint in  \Cref{eq:proof52dp} holds for each individual choice of $\bm{x}$, $\mathcal{X}$, the test with randomly sampled instances cannot be easier than the test in \Cref{eq:proof52dp} (in practice, the test will be much harder when $\bm{x}, \mathcal{X}$ are stochastic as well),  i.e.,
\begin{align}
 \forall{\bm{x^\prime}: \text{\normalfont~Test}\left(A_0; A_1(\vx^{\prime}) \right) \geq f.} 
\end{align}

We can now leverage \Cref{app_thm:worstcasestoasticcomp} which highlights that when stochastically drawing from tests that are bounded by a certain trade-off function, the composed test can also be bounded by this function, which results in 
\begin{align}
\bigotimes_{\vx^{\prime} \sim \mathcal{D}}  \text{\normalfont~Test}\left(A_0; A_1(\vx^{\prime}) \right) \geq f,
\end{align}
the definition of $f$-MIP.
\end{proof}

\subsection{On the Correspondence between $\mu$-DP and $\mu$-MIP through noisy SGD}
\label{sec:correspondence_fmip_fdp}
Our results on noisy SGD in \Cref{theorem:onestep_dp_sgd} allow to translate between the parameters $\mu_{\text{DP}}$ of $\mu$-GDP and $\mu_{\text{MIP}}$ of $\mu$-GMIP,  \emph{when our algorithm is used to guarantee these privacy notions}. This relation can be made explicit through the following Corollary: 

\begin{corollary}[Translating $\mu$ values between DP and MIP]
For one step of DP-SGD and the usual setting of $K=d$, we can convert the privacy parameters $\mu_{\text{DP}}$ and $\mu_{\text{MIP}}$ as follows:
\begin{enumerate}
    \item If the step is $\mu_{\text{DP}}$-GDP it will be $\mu_{\text{MIP}}$-GMIP with
\begin{align}
\mu_{\text{MIP}} = \min\left\{\sqrt{\frac{d}{n +\frac{4C^2}{\mu_{\text{DP}}^2}+\frac{1}{2}}},~~\mu_{\text{DP}}\right\}.
\nonumber
  \end{align}
    \item If the step is $\mu_{\text{MIP}}$-GMIP it will be $\mu_{\text{DP}}$-GDP with 
\begin{align}
    \mu_{\text{DP}} = \begin{cases}
     \frac{2}{\sqrt{\frac{d}{\mu_{\text{MIP}}^2}-n -\frac{1}{2}}}, & \text{if } \mu_{\text{MIP}} < \sqrt{\frac{2d}{2n+1}} \\
   \infty, & \text{else}.
  \end{cases}\nonumber
  \end{align}
\end{enumerate}
\end{corollary}
\begin{proof}
The result can be seen by solving for the required noise level $\tau$. When guaranteeing $\mu_{\text{DP}}$-GDP, we have
\begin{align}
    \mu_{\text{DP}}=\frac{1}{\sigma}=\frac{2C}{n\tau} \label{eq:conversionsgddp}
\end{align}
by \cite{dong2022gaussian}.
Additionally, for MIP ($K=d$) we have 
\begin{align}
\mu_{\text{MIP}} = \frac{2n_{\text{effective}}d}{n_{\text{effective}} \sqrt{2d + 4 n_{\text{effective}}d}} =\sqrt{\frac{2d}{2 n_{\text{effective}}+1}}.\label{eq:conversionsgdgdp}
\end{align}
\textbf{(1)} We can solve \Cref{eq:conversionsgddp} for $\tau$ arriving at
\begin{align}
    \tau = \frac{2C}{n\mu_{\text{DP}}}
\end{align}
and plug it into \Cref{eq:conversionsgdgdp} to arrive at the first term of the $\min$ in statement 1. We note that due to \Cref{thm:mipweakerthandp}, the level of $\mu_{\text{MIP}}$ cannot be higher than the level of $\mu_{\text{DP}}$ and thus take the min of both terms.

\textbf{(2)} To arrive at statement 2, we perform the opposite conversion and solve \Cref{eq:conversionsgdgdp} for $\tau$. Solving for $n_{\text{effective}}$, we obtain
\begin{align}
n_{\text{effective}} &= n+\frac{n^2\tau^2}{C^2} = \frac{d}{\mu_{\text{MIP}}^2}-\frac{1}{2}.
\end{align}
As we have $\tau^2>0$ and therefore require $n_{\text{effective}}>n$, no solution exists for $\mu_{\text{MIP}}^2<\frac{2d}{2n+1}$. For the remaining values, we obtain
\begin{align}
\tau^2 &= C^2\frac{\frac{d}{\mu_{\text{MIP}}^2}-\frac{1}{2}-n}{n^2}
\end{align}
which, plugged into \Cref{eq:conversionsgddp}, gives rise to the result in statement 2.
\end{proof}